\documentclass[10pt]{article}
\usepackage[nohead,margin=1.0in]{geometry}
\usepackage[dvipsnames]{xcolor}
\usepackage{amsmath,amssymb,amsthm,mathtools}
\usepackage{enumitem}
\usepackage[backend=biber,
  style=numeric,
  sorting=none,          
  maxnames=3,
  giveninits=true,       
  citestyle=numeric-comp
  ]{biblatex}
\usepackage{nicefrac}
\usepackage{hyperref}

\usepackage{bbm}
\usepackage[dvipsnames]{xcolor}
\usepackage{booktabs}
\usepackage{graphicx,caption,subcaption}
\usepackage{url}
\usepackage[normalem]{ulem}
\usepackage{tikz}
\usetikzlibrary{decorations.pathreplacing}
\usepackage{pgfplots}
\usetikzlibrary{pgfplots.groupplots}
\usepackage{multicol, multirow}
\usepackage{array} 
\usepackage{booktabs}
\usepackage{transparent}

\pgfplotsset{compat=1.18}

\theoremstyle{plain}
\newtheorem{theorem}{Theorem}
\newtheorem{lemma}[theorem]{Lemma} 

\newtheorem{corollary}[theorem]{Corollary}

\newtheorem{remark}[theorem]{Remark}

\newcommand{\N}{\mathbb{N}}
\newcommand{\R}{\mathbb{R}}

\newcommand{\E}{\mathbb{E}}
\newcommand{\Id}{\mathrm{Id}}
\newcommand{\supp}{\mathrm{supp}}
\newcommand{\Lip}{\mathrm{Lip}}
\renewcommand{\d}{\mathrm{d}}

\newcommand{\ptarget}{p_H}
\newcommand{\gaussian}{p_Z}
\newcommand{\DKL}[2]{D_{\mathrm{KL}}\!\left(#1 \,\middle\|\, #2\right)}
\newcommand{\Cov}{\mathrm{Cov}}
\newcommand{\Var}{\mathrm{Var}}

\begin{document}

\title{Expressivity of Bi-Lipschitz Normalizing Flows:\\ A Score-Based Diffusion Perspective}

\author{
Meira Iske\thanks{\parbox[t]{\linewidth}{Center for Industrial Mathematics, University of Bremen, Bremen, Germany,\\
Email: \texttt{iskem@uni-bremen.de}}}
\and
Carola-Bibiane Sch{\"o}nlieb\thanks{\parbox[t]{\linewidth}{Department of Theoretical Physics and Applied Mathematics, University of Cambridge, Cambridge, U.K.,\\
Email: \texttt{cbs31@cam.ac.uk}}}
}

\date{}

\maketitle

\begin{abstract}
Many normalizing flow architectures impose regularity constraints, yet their distributional approximation properties are not fully characterized. We study the expressivity of bi-Lipschitz normalizing flows through the lens of score-based diffusion models. For the probability flow ODE of a variance-preserving diffusion, Lipschitz regularity of the score induces a flow of bi-Lipschitz diffeomorphic transport maps. This ODE bridge allows us to analyze the distributional approximation power of bi-Lipschitz normalizing flows and, conversely, derive deterministic convergence guarantees for diffusion-based transport. 
Our key idea is to use the probability flow ODE to link regularity of the score to regularity of the induced transport maps.
We verify score regularity for broad target densities, including compactly supported densities, Gaussian convolutions of compactly supported measures and finite Gaussian mixtures. 
We obtain a universal distributional approximation result: Gaussian pullbacks induced by bi-Lipschitz variance-preserving transport maps are $L^1$-dense among all probability densities. For Gaussian convolution targets, we further obtain convergence in Kullback–Leibler divergence without early stopping. 
\end{abstract}

\section{Introduction}
Generative modeling aims to learn probability distributions from data in a way that enables efficient sampling and likelihood evaluation. Over the past decade, such models have become central to deep learning, with applications ranging from image and audio generation~\cite{Kingma2018, Roberts2018, Kumar2019} to inverse problems and noise modeling~\cite{Chung2025, Abdelhamed2019, Iske2025}. Prominent approaches include variational autoencoders~\cite{Kingma2014}, generative adversarial networks~\cite{Goodfellow2014}, normalizing flows~\cite{Rezende2015}, and diffusion-based models~\cite{Song2021, Sohldickstein2015}. Among these, normalizing flows and diffusion models are particularly attractive due to their transparent probabilistic formulation and access to data likelihoods~\cite{Papamakarios2021, Song2021}.

Normalizing flows represent a target distribution as the pushforward of a simple latent distribution under an invertible transport map~\cite{Dinh2015}. In practice, these transport maps are parameterized by invertible neural networks and are often equipped with regularity constraints. Bi-Lipschitz architectures, such as invertible residual networks~(iResNets) based on contractive residual blocks or additive coupling architectures, provide stability by controlling both the forward map and its inverse. They therefore prevent arbitrary amplification of perturbations between data and latent space. Consequently, these architectures yield bounded Jacobian determinants and hence controlled local volume distortion. Similar properties arise in less restrictive settings, for instance in coupling-based flows with bounded scale components, where the Jacobian determinant is bounded without necessarily imposing global bi-Lipschitz regularity. While such constraints ensure stability and invertibility~\cite{Behrmann2019, Arndt2023}, they may limit expressivity, leaving the trade-off between stability and approximation power only partially understood.

\begin{figure}[t]
\centering
\def\svgwidth{0.85\columnwidth}
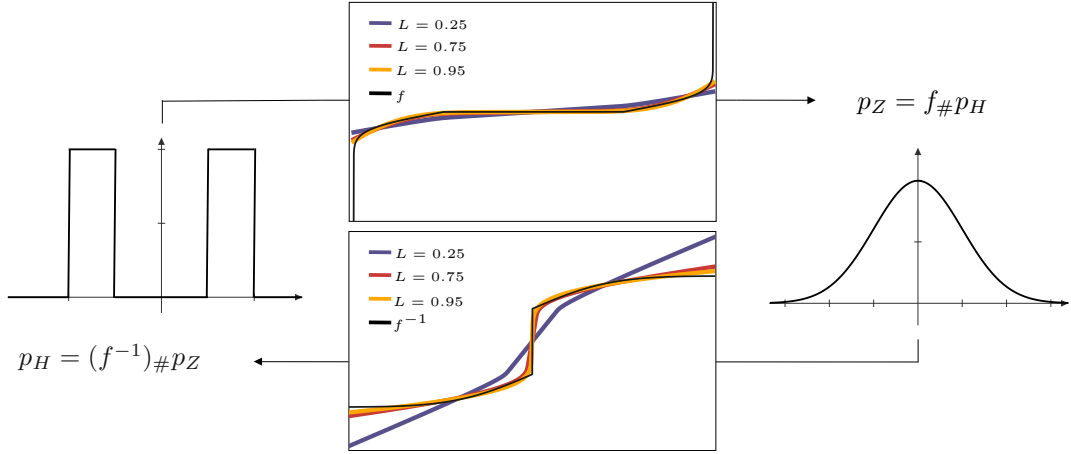
\caption{
    Example of a compactly supported pdf $\ptarget$ and monotone transport $f$ to the standard Gaussian $\gaussian$, satisfying $\gaussian = f_\# \ptarget$. The exact transport is not a global diffeomorphism on $\R$, and neither $f$ nor $f^{-1}$ is globally Lipschitz. Its discontinuouities can therefore only be approximated by bi-Lipschitz maps. A bi-Lipschitz map cannot exactly transform the connected full support of the Gaussian into the disconnected compact support of $\ptarget$. The coloured transport maps correspond to trained iResNets at different Lipschitz constraints, depending on $L$. Larger $L$ corresponds to a larger feasible Lipschitz bound of the network. Implementation details can be found in Section~\ref{sec:num_examples}.}
    \label{fig:illustration}
\end{figure}

The existence of transport maps between probability measures is classical. The Knothe-Rosenblatt rearrangement yields a monotone transport~\cite{Rosenblatt1952, Knothe1957}, while Brenier’s theorem characterizes optimal transport for quadratic cost~\cite{Brenier1991}. Under strong structural assumptions such as log-concavity or bounded perturbations of Gaussian measures, the optimal transport map is Lipschitz~\cite{Caffarelli2000}. Relatedly, the Dacorogna-Moser theorem provides an existence theory for diffeomorphisms satisfying volume form matching equations on bounded domains under boundary conditions~\cite{DacorognaMoser1990}. Outside these regimes, however, global Lipschitz regularity may be violated, and counterexamples show that bi-Lipschitz transports do not exist for general target-latent pairs~\cite{Cornish2020}. Correspondingly, recent work has identified expressivity limitations of normalizing flows under uniformly bounded Lipschitz constraints~\cite{Kong2020, Verine2023}. A simple example illustrating this phenomenon is shown in Figure~\ref{fig:illustration}, where a target distribution with probability density function~(pdf) $\ptarget$ admits a monotone transport to the standard Gaussian that is not globally Lipschitz. In particular, no exact bi-Lipschitz transport exists for this target-latent pair.

Nevertheless, such transports can often be approximated arbitrarily well by maps with sufficiently large Lipschitz constants. However, it is less clear whether such map-level approximation transfers to approximation of the induced densities in strong metrics. This leads to the question whether the limitations concerning distribution approximation are intrinsic to bi-Lipschitz regularity or arise only from imposing a \emph{uniform} bound on the Lipschitz constants. Answering this distinguishes fundamental limitations from quantitative effects and clarifies whether stability and invertibility remain compatible with strong distributional approximation. Our results show that the known limitations are not intrinsic to bi-Lipschitz regularity itself, but arise from imposing uniform Lipschitz bounds. We prove full expressivity with respect to the $L^1$-norm for specific candidates of bi-Lipschitz transport maps obtained from diffusion flows with non-uniform Lipschitz bounds.

Diffusion models have recently emerged as a powerful and empirically successful generative modeling approach~\cite{Sohldickstein2015, Ho2020}. In score-based diffusion models~(SDMs), initial target data is gradually perturbed by a stochastic process and samples are generated via the corresponding reverse-time SDE using an estimated score function~\cite{Song2021}. Importantly, diffusion models admit an associated probability flow ODE, which induces a deterministic family of transport maps between intermediate distributions of this process at selected times. Fixing a finite time horizon yields an end-to-end transport map that can be viewed as a normalizing flow. This observation naturally connects diffusion models to continuous normalizing flows~\cite{Chen2018}. Since the diffusion provides a controlled smoothing path, the time-dependent gradient field in the probability flow ODE can be analyzed and used to transfer score regularity to stability of the induced transport maps. 
This renders the regularity of the transport map amenable to analysis through the underlying time-dependent vector field. Recent theoretical work shows that such score regularity is essential for stability and convergence guarantees under various assumptions on the target distribution~\cite{Chen2023, Chen2023b, Lee2023, Conforti2025}.

\textbf{Contribution:} In this work, we leverage the probability flow ODE of score-based diffusion models as a tractable transport family to study the distributional expressivity of bi-Lipschitz normalizing flows. 
We focus on the regularity regime of $C^1$ bi-Lipschitz transports, motivated by specific normalizing flow architectures, such as iResNets. Since bi-Lipschitz $C^1$-transport maps form a subclass of transports with bounded Jacobian determinant, our existence-based approximation results immediately also apply to this larger class. Considering variance-preserving~(VP) diffusions, we show that for broad classes of target distributions, the induced transport maps are bi-Lipschitz, with Lipschitz constants depending on the time horizon. Allowing for non-uniform time-dependent Lipschitz bounds yields full expressivity: for every target density in $L^1$, there exists a bi-Lipschitz $C^1$ transport map induced by a VP diffusion process, whose Gaussian pullback approximates the target arbitrarily well in $L^1$. Our results should therefore be understood as existence-based results on distributional expressivity, formulated for VP probability-flow transports rather than for a fixed network architecture. They show that bi-Lipschitz regularity itself is not an obstruction to $L^1$-approximation, provided that the Lipschitz constants of the approximating transports are allowed to depend on the approximation accuracy. 
We further establish KL convergence results, including convergence without early stopping for a dense subset of $L^1$ densities. Our analysis reveals a direct link between score regularity along diffusion processes and distributional universality of bi-Lipschitz normalizing flows. We identify several practically relevant target classes, including compactly supported densities, log-concave densities, and finite Gaussian mixtures, for which the required score regularity holds, and show that, under certain assumptions, these expressivity guarantees are robust to score approximation.
In contrast to classical existence results of transports between target-latent pairs such as Dacorogna–Moser, our framework yields an explicit characterization of such transport maps together with distributional approximation guarantees. For certain target classes, this additionally reveals stronger convergence metrics than $L^1$. A detailed summary of our contributions is given in Section~\ref{subsec:contributions}.

\textbf{Organization:} This work is structured as follows. Section~\ref{sec:motivation} introduces the normalizing flow and diffusion frameworks. Section~\ref{subsec:related_literature} reviews related literature, and Section~\ref{subsec:contributions} summarizes our main contributions. Section~\ref{sec:expressivity} presents the core expressivity results and identifies target distribution classes satisfying the required regularity assumptions. Section~\ref{sec:connection_learned} extends the analysis to learned score functions, and Section~\ref{sec:num_examples} illustrates the theoretical findings with numerical experiments.

\section{Motivation, Setting and Contributions}~\label{sec:motivation}

Normalizing flows model a target density $\ptarget$ as the pushforward of a simple latent distribution $\gaussian$ under a diffeomorphism $\varphi:\R^n\to\R^n$. Throughout, we fix $\gaussian=\mathcal N(0,I_n)$ with density ${\gaussian:\R^n\to[0,\infty)}$. For $z\sim\gaussian$, the model output $x = \varphi^{-1}(z)$ has density ${\ptarget: \R^n \rightarrow [0,\infty)}$, which is given by the change-of-variables formula
\begin{equation}\label{eq:CoV}
    \ptarget(x)
    = \gaussian(\varphi(x))\,\bigl|\det J_\varphi(x)\bigr|,
    \quad x\in\R^n,
\end{equation}
where $J_\varphi(x)$ denotes the Jacobian of $\varphi$. Equivalently,
\begin{equation*}
    \ptarget = (\varphi^{-1})_\# \gaussian \quad\text{and}\quad
    \gaussian = \varphi_\# \ptarget,
\end{equation*}
where we identify probability measures with their densities and $g_\#\mu$ denotes the pushforward of $\mu$ under $g$. When we restrict $\varphi$ to the class of bi-Lipschitz diffeomorphisms, there exist constants ${0 < c \leq C < \infty}$ such that
\begin{equation*}
    c \|x-y\| \leq \| \varphi(x) -\varphi(y) \| \leq C \|x-y\| \quad \text{for all } x,y \in \R^n,
\end{equation*}
which implies
\begin{equation*}
    c^n \leq |\det J_\varphi(x)|\leq C^n \quad \text{for all } x \in \R^n.
\end{equation*} 
Such bi-Lipschitz constraints provide stability and invertibility guarantees, but are often thought to come at the price of reduced expressivity. To study their expressivity, we adopt a viewpoint that encompasses both score-based diffusion models~\cite{Song2021} and end-to-end normalizing flows as in~\eqref{eq:CoV}. Unlike end-to-end flows, score-based diffusion models define a time-indexed family of marginals $(p_t)_{t\in[0,T]}$ via a forward diffusion and admit an associated probability flow ODE. This constitutes a continuous normalizing flow whose deterministic solution maps transport $p_s$ to $p_t$ for $0\leq s\leq t\leq T$.

\subsection{Score-Based Diffusion Process}

Let $(X_t)_{t\in[0,T]}$ be governed by the It{\^o} SDE
\begin{equation}\label{eq:SDE}
    \d X_t = f(X_t,t)\,\d t + g(t) \, \d W_t, \quad X_0 \sim \ptarget,
\end{equation}
where $(W_t)_{t\in[0,T]}$ is a standard $n$-dimensional Brownian motion, $f:\R^n \times [0,T] \longrightarrow \R^n$ denotes the \emph{drift coefficient} and $g: [0,T] \longrightarrow (0, \infty)$ the \emph{diffusion coefficient}. We denote by $p_t$ the marginal density of $X_t$ at time $t$ and define its \emph{score function} by 
\begin{equation*}
    s_t(x) = \nabla_x \log p_t(x).
\end{equation*}

\textbf{Reverse-time SDE.} The time-reversed process $(\bar X_t)_ {t\in[0,T]}$ satisfies the SDE
\begin{equation}\label{eq:reverse_SDE}
    \d\bar X_t = \left[f(\bar X_t,t) - g(t)^2s_t(\bar X_t)\right]\,\d t + g(t) \, \d \bar W_t,
\end{equation}
where $(\bar W_t)_{t\in[0,T]}$ is again a standard $n$-dimensional Brownian motion. This SDE transports $p_T \approx \gaussian$ back toward $\ptarget$ when initialized at $t= T$ for $T$ reasonably large~\cite{Song2021}.

\textbf{Probability flow ODE.} The deterministic flow sharing the same marginals $(p_t)_{t \in [0,T]}$ as~\eqref{eq:reverse_SDE} is given by the \emph{probability flow ODE}
\begin{equation}\label{eq:prob_ODE}
    \frac{\d x}{\d t} = f(x,t) - \frac{1}{2}g(t)^2 s_t(x).
\end{equation}
Let $\varphi_{s \to t}:\R^n \to \R^n$ denote the solution transport map of~\eqref{eq:prob_ODE}, mapping an initial condition at time $s$ to its value at time $t$. Then, for any random variable $X_s \sim p_s$, we have  $p_t = \varphi_{s \to t}{}_\# p_s$. In particular, $\varphi_{0 \to T}$ is a transport map from $\ptarget = p_0$ to $p_T \approx \gaussian$. Thus, $(\varphi_{s \to t})_{0 \leq s \leq t \leq T}$ forms a flow of transports. We will show that, under suitable conditions, $\varphi_{0 \to T}$ is in fact a bi-Lipschitz diffeomorphism and we will consider this probability flow transport map to analyze the expressivity of both bi-Lipschitz normalizing flows and one specific score-based diffusion model.

\subsection{The Variance-Preserving~(VP) Process}

In the following, we specialize the general score-based diffusion framework to the variance-preserving~(VP) diffusion process, in which the signal is progressively damped and isotropic noise is added in a balanced way that keeps the overall variance unchanged. The corresponding forward SDE in its general VP form is
\begin{equation}\label{eq:VP-SDE}
    \d X_t = -\frac{1}{2}\beta(t) X_t\, \d t + \sqrt{\beta(t)} \,\d W_t, \quad X_0 \sim \ptarget,
\end{equation}
for $\beta: [0,T] \longrightarrow (0,\infty)$, which corresponds to the choice $f(X_t,t) = -\frac{1}{2}\beta(t) X_t$ and $g(t) = \sqrt{\beta(t)}$ in~\eqref{eq:SDE}. 

In this work, we restrict to the VP process with a constant diffusion rate $\beta(t)\equiv1$, so that~\eqref{eq:VP-SDE} simplifies to
\begin{equation}\label{eq:VP-SDE_fixedbeta}
    \d X_t = -\frac{1}{2}X_t\, \d t + \d W_t, \quad X_0 \sim \ptarget.
\end{equation}

Define  
\begin{equation}\label{eq:a_and_sigma}
    a(t) = e^{-t/2} \quad \text{and} \quad \sigma^2(t) = 1-a^2(t) = 1-e^{-t}.
\end{equation}

Then $X_t = a(t)X_0+ \sigma(t)Z$ with $Z\sim \gaussian$ so that $X_t | X_0 \sim \mathcal{N}(a(t)X_0, \sigma^2(t) I_n)$. The %(Mehler) 
representation of the marginal density of $X_t \sim p_t$ is
\begin{equation}\label{eq:p_t}
    p_t(x)=\int_{\R^n} \ptarget(y)\,\phi_{\sigma^2(t),0} \big(x-a(t)\,y\big)\,\d y,
\end{equation}
where $\phi_{\sigma^2(t),0} = \mathcal{N}(0,\sigma^2(t) I_n)$ (see e.g.~\citeauthor{BakryGentilLedoux2014}, \citeyear{BakryGentilLedoux2014}). 
As $t \to \infty$ we have $a(t)\rightarrow 0$ and $\sigma^2(t)\rightarrow 1$, so $p_t$ converges to $\gaussian$. The joint distribution of $(Y,X):=(X_0,X_t)$ has density $p_{Y,X}(y,x)=\ptarget(y)\,\phi_{\sigma^2(t),a(t)y}(x)$ and its posterior density of $Y$ given $X=x$ is
\begin{equation}\label{eq:p_t,x}
    p_{t,x}(y):=p_{Y\mid X=x}(y)
    =\frac{\ptarget(y)\phi_{\sigma^2(t),a(t)y}(x)}
           {\int_{\mathbb{R}^n} \ptarget(z)\phi_{\sigma^2(t),a(t)z}(x)\,\d z}.
\end{equation}

For the VP process, the corresponding probability flow ODE becomes
\begin{equation}\label{eq:VP-flow}
    \frac{\d x}{\d t} = v_t(x) \quad \text{with} \quad v_t(x) := -\frac{1}{2}x - \frac{1}{2}s_t(x).
\end{equation}

\textbf{Notation.}
Throughout, let $n$ be the dimension of the data space. Furthermore, $\|\cdot\|$ denotes the Euclidean norm on $\R^n$, and for matrices the corresponding operator norm, while $\|\cdot\|_{L^1}$ denotes the $L^1$ norm on $\R^n$ with respect to the Lebesgue measure. Let $\gaussian$ denote standard Gaussian density and $\ptarget$ the probability density function~(pdf) of the target. We use $\phi_{\sigma^2, \mu}(x) := (2\pi\sigma^2)^{-n/2} \exp\left( -\tfrac{\| x - \mu \|^2}{2\sigma^2}\right)$ to denote the Gaussian density with mean $\mu$ and covariance $\sigma^2 I_n$ and define $C_b^k(\R^n):=\{ f:\R^n \rightarrow \R \, : \, f\in C^k, \, \nabla^j f \in L^\infty, \, j=0,\hdots, k \, \}$. For symmetric matrices $A, B \in \R^{n \times n}$, $A \succeq B$ denotes the L{\"o}wner order, i.e. $x^\top (A-B) x \geq 0$ for all $x \in \R^n$, likewise $A \succ B$ is used, whenever $x^\top (A-B) x > 0$ for $x \neq 0$. 

\subsection{Related Literature}~\label{subsec:related_literature}

This work lies at the intersection of invertible generative models and diffusion-based approaches. We review a related expressivity result and subsequently discuss existing literature on the expressive power of end-to-end normalizing flows and on the convergence theory of score-based diffusion models, with a focus on distributional approximation and regularity of the induced transport maps.

\textbf{Volume Forms.}
In the context of nonlinear PDEs,~\cite{DacorognaMoser1990} prove existence results for diffeomorphisms $\varphi$ that transform one sufficiently smooth, positive density~(volume form) into another on a bounded domain. Accordingly, $\varphi$ is bi-Lipschitz on this domain. These results require compatibility conditions, such as equal total mass, and impose boundary conditions such as $\varphi=\mathrm{id}$ on $\partial\Omega$. This suggests a possible heuristic route to $L^1$-approximation on $\mathbb R^n$: for each approximation accuracy, one may smooth the target density, blend it with a latent density such as the Gaussian so that the tails agree outside a large bounded domain, and then apply the bounded-domain result of~\cite{DacorognaMoser1990} to the resulting approximating densities. If the corresponding transports can be glued to the identity outside these bounded domains, this would yield global bi-Lipschitz transports whose pullbacks approximate the target density. The construction is generally non-uniform in the sense that the Lipschitz constants may depend on the approximation accuracy. However, this route requires additional compatibility arguments. The identity extension need not be globally $C^1$ unless the map already agrees with the identity in a boundary region.

\textbf{End-to-end Normalizing Flows.} 
Substantial theoretical literature studies the approximation capabilities of invertible architectures at the transport map level. Universality results show that coupling-based normalizing flows can approximate diffeomorphisms under suitable architectural assumptions~\cite{Teshima2020, Ishikawa2023}. Related work on neural ODEs and invertible residual networks establishes positive and negative results on which classes of transformations can be represented~\cite{Chen2018, Zhang2020, Ishikawa2023}. More recently, \cite{Jin2024} explicitly study the approximation of bi-Lipschitz maps by coupling-based invertible neural networks. While these results clarify the representational power of invertible architectures, they are primarily concerned with function approximation and do not directly address approximation of probability distributions in strong metrics such as KL/TV convergence. 
Beyond map-level expressivity, several works investigate distributional approximation properties of normalizing flows. Positive results include convergence guarantees in weak probability metrics, most notably the work of~\cite{Kong2021} showing that iResNets are universal distribution approximators with respect to maximum mean discrepancy~(MMD) under suitable depth and architectural conditions. In \cite{Draxler2024}, the expressive power of affine coupling networks under volume-preserving and volume-bounded constraints is analyzed, clarifying how control of volume distortion influences approximation capabilities. The results of~\cite{Draxler2024} provide evidence that volume-controlled flows, including bi-Lipschitz flows, can be expressive in a distributional sense weaker than KL divergence, when subsequently increasing the number of coupling blocks.

A complementary line of work highlights limitations under fixed (bi-)Lipschitz constraints. The work of~\cite{Verine2023} derives explicit lower bounds in KL divergence or TV distance, showing that certain target distributions cannot be approximated arbitrarily well by pushforward models with uniformly bounded Lipschitz constants. \cite{Cornish2020} prove that for some target-base pairs any exact transport must necessarily violate global Lipschitz bounds. Related analyses by~\cite{Kong2020} further provide explicit examples where classical flow families fail to approximate target distributions. Collectively, these works delineate the trade-off between stability and expressivity under uniform regularity constraints, but they do not address whether strong distributional approximation is possible under non-uniform Lipschitz constants.

While some of the above works study expressivity for concrete network architectures, our analysis is carried out at the level of transport maps as functions, hence enabling to transfer the conclusions to common normalizing flow models.

\textbf{Score-based Diffusion Models.}
Score-based diffusion models form a central class of modern generative models, with a likelihood-based formulation grounded in score matching and reverse-time stochastic differential equations~\cite{Song2021}. A key distinction in this literature is between stochastic SDE-based samplers and the associated deterministic probability flow ODE. Most theoretical analyses focus on the SDE formulation and establish convergence guarantees and rates under assumptions on the accuracy of a learned score function, typically measured in an $L^2$ sense and often combined with discretization effects or early stopping~\cite{Chen2023b,Chen2023,Lee2023,Conforti2025,Benton2024}.

Several contributions analyze diffusion models under assumed score regularity, establishing Wasserstein, KL, or TV convergence under global Lipschitz or semi convexity assumptions on the score or the log-density, often together with structural conditions such as log-concavity or sub-Gaussian tails~\cite{Lee2023,Yu2025,Yang2022}. A smaller subset derives convergence results directly for probability flow ODE samplers~\cite{Gao2025,Chen2023a}, while related work analyzes improved schemes such as prediction--correction within the SDE framework~\cite{Pedrotti2024}.
 
\cite{Chen2023} establish convergence rates in KL divergence, TV, and Wasserstein distances by relying on a finite second moment assumption together with either early stopping or by using regularity assumptions on the score. This extends the earlier work of~\cite{Chen2023b}, which additionally covers early-stopped TV convergence and non-early stopped TV convergence under score regularity explicitly in the case of compactly supported densities. Furthermore, \cite{Conforti2025} prove KL convergence under minimal data assumptions, including finite Fisher information, without early stopping. These three works are formulated within the SDE framework and analyze discretized reverse-time dynamics. However, \cite{Conforti2025} focus on stochastic convergence guarantees rather than regularity of the induced probability flow ODE, and \cite{Chen2023} assumes score smoothness in one regime of the analysis. 

More recently,~\cite{Brigati2025,Mooney2024,Stephanovitch2025} derive score regularity from structural properties of the data distribution. Using heat-flow arguments, log-concavity, or related functional inequalities, these works establish Lipschitz or one-sided Lipschitz bounds for the score over finite time horizons and clarify the inherently time-dependent nature of score regularity along the diffusion. While~\cite{Brigati2025,Stephanovitch2025} primarily address well-posedness and stability of the associated flows without explicitly using the score regularity to obtain distribution-level approximation or expressivity results,~\cite{Mooney2024} additionally establishes KL convergence of the generated distributions under near log-concavity or more general smoothness and tail assumptions.

Altogether, these works leave open whether distributional approximation of general targets can be expressed by deterministic, bi-Lipschitz Gaussian pullbacks induced by score-based diffusion processes. Furthermore, it remains unclear under what conditions the potential transport constructions based on~\cite{DacorognaMoser1990} as described above, admit approximation guarantees stronger than $L^1$-density. 

\subsection{Our Contributions}\label{subsec:contributions}

This work establishes a connection between regularity theory for the VP diffusion score and an expressivity theory for bi-Lipschitz normalizing flows: Depending on score regularity, the VP probability flow induces a deterministic family of bi-Lipschitz normalizing flows, which yield distributional approximation guarantees. Our analysis follows the principle that the strength of distributional approximation guarantees is governed by the regularity of the score. In particular, we distinguish two regimes:
\begin{itemize}
    \item[(i)] \textbf{uniform-in-space regularity,} which yields convergence in $L^1$, 
    \item[(ii)] \textbf{uniform-in-time and uniform-in-space regularity,} which yields convergence in KL divergence.
\end{itemize}
While $L^1$-convergence guarantees approximation in total variation, KL convergence additionally controls likelihoods and is therefore directly aligned with likelihood-based generative modeling. Our main contributions are as follows:
\begin{itemize}
    \item \textbf{Bi-Lipschitz VP transport maps and expressivity of NFs.}
    In Theorem~\ref{thm:score} we show that the implication contained in (i) holds for all $\ptarget$ with finite second moment, from a bi-Lipschitz flow ODE perspective. In particular, the VP probability flow ODE~\eqref{eq:VP-flow} induces bi-Lipschitz diffeomorphisms whose inverses push the standard Gaussian arbitrarily close to $\ptarget$ in $L^1$. Corollary~\ref{cor:score} establishes the implication contained in (ii), extending this result to non-early-stopped KL convergence for $\ptarget \in C_b^2(\R^n)$.
    \item \textbf{VP score regularity for rich target classes.}
    In Lemmas~\ref{lemma:pH_compact}--\ref{lemma:score_prop_conv} and Corollary~\ref{cor:score_gauss_mixture}, we identify structurally rich  
    target classes $\ptarget$ satisfying (i). This includes compactly supported densities~\eqref{eq:A1}, log-concave densities~\eqref{eq:A2}, Gaussian convolutions with compactly supported base measures~\eqref{eq:A3}, and finite Gaussian mixtures~\eqref{eq:A4}. Moreover, for~\eqref{eq:A3} and~\eqref{eq:A4}, we extend to (ii). Altogether, this shows $L^1$ convergence for $\eqref{eq:A1}, \eqref{eq:A2}$ and KL convergence for $\eqref{eq:A3}, \eqref{eq:A4}$, extending the results of~\cite{Chen2023b, Chen2023}. In contrast to~\cite{Conforti2025}, which assumes finite Fisher information with respect to the standard Gaussian distribution, our $L^1$ convergence results for compactly supported targets~\eqref{eq:A1} do not rely on this condition. The convergence hierarchy is reflected in Table~\ref{tab:vp-summary}, where stronger score regularity leads to stronger convergence guarantees. 
    \item \textbf{Universality for arbitrary densities.} 
    Corollary~\ref{cor:universal_vp} shows that Gaussian pullbacks induced by bi-Lipschitz VP transport maps are $L^1$-dense in the class of pdfs in $L^1$. Consequently, for every $\ptarget \in  L^1$, there exists an approximant $p \in L^1$ which satisfies (i) or (ii). In particular, we show that for $p$, (ii) is satisfied. The last column of Table~\ref{tab:vp-summary} corresponds to this general class of pdfs without explicit score regularity guarantee. 
    \item \textbf{Expressivity with learned scores and bi-Lipschitz normalizing flows.} 
    By assuming $L^2$ score approximation of a learned flow, we prove in Theorem~\ref{thm:learned_score} that the induced learned VP probability flow still yields bi-Lipschitz transports whose Gaussian pullbacks approximate $\ptarget$ in $L^1$. For this result we need to assume that the learned flow ODE shares the same marginal densities as the learned reverse-time SDE. This provides a conceptual connection between learned scores and expressivity. Addressing practical training aspects and weaker assumptions is left for future work.
\end{itemize}

\begin{table}[!t]
    \centering
    \footnotesize
    \renewcommand{\arraystretch}{1.05}
    \setlength{\tabcolsep}{5pt}
    \newcolumntype{L}[1]{>{\raggedright\arraybackslash}m{#1}}

    \begin{tabular}{p{0.32\textwidth}p{0.30\textwidth}p{0.30\textwidth}}
        \toprule
        \multirow{2}{=}{Assumption on $\ptarget$}
        & \multirow{2}{=}{VP score regularity ($\delta > 0$)} 
        & Main convergence guarantees \\
        &
        &
        (bi-Lipschitz VP transport) \\
        \midrule
        \textbf{Compact support~\eqref{eq:A1}}\par\vspace{1.5mm}
        $\supp(\ptarget)\subseteq K, \ K \subset \R^n$ compact 
        & Lemma~\ref{lemma:pH_compact}:\par\vspace{1mm}
        $\sup\limits_{x\in\R^n}\|\nabla s_t(x)\|\leq L(t), \ t\in[\delta,T]$

        & Corollary~\ref{cor:approx_prop}:\par\vspace{1mm}
        $\big\|\ptarget - (\varphi_{\delta\to T}^{-1})_\#\gaussian\big\|_{L^1}
            \xrightarrow[\,T\uparrow\infty]{\delta\downarrow 0} 0,$\par\vspace{1mm}
        $\DKL{p_\delta}{(\varphi_{\delta\to T}^{-1})_\# \gaussian}
            \xrightarrow[\,T\uparrow\infty]{} 0,$\par\vspace{1mm}
        \\[5mm] 
        \addlinespace[2mm]
        \hline \\ \addlinespace[-2.5mm]
        \textbf{Log-concave~\eqref{eq:A2}}\par\vspace{1.5mm}
        $\ptarget(x)\propto e^{-V(x)}, \ \nabla^2 V(x)\succeq 0$ 
        & Lemma~\ref{lemma:score_log_concave}:\par\vspace{1mm}
        $\sup\limits_{x\in\R^n}\|\nabla s_t(x)\|\leq L(t), \ t\in[\delta,T]$
        & \multirow{3}{=}{convergence as for compact-support~\eqref{eq:A1}}
        \\[10mm] \hline \\ \addlinespace[-2.5mm]
        \textbf{Gaussian convolution~\eqref{eq:A3}}\par\vspace{1.5mm}
        $\ptarget = \phi_{\Sigma,0} \ast \mu \in \mathcal{G}$, see $\eqref{eq:G},$ \par \vspace{1mm}
        $\supp(\mu)\subset B(0,R)$ 
        & Lemma~\ref{lemma:score_prop_conv}:\par\vspace{1mm}
        $\sup\limits_{t\in[0,T]}\sup\limits_{x\in\R^n}\|\nabla s_t(x)\|
            \leq M_T$
        & Corollary~\ref{cor:approx_prop}: \par\vspace{1mm}
        $\big\|\ptarget - (\varphi_{0\to T}^{-1})_\#\gaussian\big\|_{L^1}
            \xrightarrow[\,T\uparrow\infty]{} 0,$ \par\vspace{1mm}
        $\DKL{\ptarget}{(\varphi_{0\to T}^{-1})_\# \gaussian}
            \xrightarrow[\,T\uparrow\infty]{} 0,$\par\vspace{1mm}
        \\[5mm] 
        \addlinespace[2mm]
        \hline \\ \addlinespace[-2.5mm]
        \textbf{Finite Gaussian mixture~\eqref{eq:A4}}\par\vspace{1.5mm}
        $\ptarget(x)=\sum_{k=1}^K \alpha_k \phi_{\Sigma,m_k}(x)$ 
        & Corollary~\ref{cor:score_gauss_mixture}:\par\vspace{1mm}
        $ \sup\limits_{t\in[0,T]}\sup\limits_{x\in\R^n}\|\nabla s_t(x)\|
            \leq M_T$
        &\multirow{3}{=}{convergence as for Gaussian convolution~\eqref{eq:A3}}
        \\[10mm] \hline \\ \addlinespace[-2.5mm]
        \multirow{3}{=}{$\ptarget \in L^1(\R^n)$ a pdf} 
        & \multirow{3}{=}{\centering{\textbf{---}}}
        & 
        Corollary~\ref{cor:universal_vp}: \par\vspace{1mm} $\forall \varepsilon > 0: \, \exists \, \varphi_\varepsilon$ bi-Lipschitz with \par\vspace{1mm}
        $\big\|\ptarget - (\varphi_\varepsilon^{-1})_\#\gaussian\big\|_{L^1} < \varepsilon$ \\
        \bottomrule
    \end{tabular} 
    \caption{Summary of score properties and convergence results. For each structural assumption on $\ptarget$ we state the corresponding score regularity and the resulting approximation guarantee obtained via the probability flow ODE~\eqref{eq:VP-flow}. The ODE induced transport maps with initial density $\ptarget$ as in the given assumption are denoted by $\varphi_{\delta \to T}$, while $\varphi_\varepsilon$ denotes a general ODE induced transport initialized at some $p \in \mathcal{G}$, as described in Corollary~\ref{cor:universal_vp}.}
    \label{tab:vp-summary}
\end{table}

\section{Main Results}\label{sec:expressivity}

In this section we state our main expressivity and regularity results for the VP probability flow~\eqref{eq:VP-flow}, eventually leading to a universal approximation result for the induced VP transport maps for all probability densities. All proofs of this section are collected in Appendix~\ref{sec:appendix_proofs}.

First, we impose regularity assumptions on the VP score to show that the corresponding transport maps are bi-Lipschitz $C^1$-diffeomorphisms and investigate their approximation properties in the latent and target space.

\begin{theorem}\label{thm:score}
    Let $\ptarget :\R^n \longrightarrow [0,\infty)$ be a pdf with $\E_{\ptarget}[\|X\|^2] < \infty$. Furthermore, let $p_t$ denote the marginals of the VP process~\eqref{eq:VP-SDE_fixedbeta}. Suppose there exists a function ${L:[0,\infty) \rightarrow [0,\infty)}$ such that the score $s_t(x)=\nabla_x \log p_t(x)$ satisfies
    \begin{itemize}
        \item[(i)] $\sup_{x\in \R^n} \| \nabla s_t(x)\| \leq L(t) < \infty$ for all $t > 0$,
        \item[(ii)] $\int_{\delta}^T (1+L(t))\,\d t < \infty $ for all $0 < \delta < T < \infty$.
    \end{itemize}
    Let $(\varphi_{\delta \to T})_{0 < \delta < T < \infty}$ be the flow defined by~\eqref{eq:VP-flow}. Then, the following properties hold:
    \begin{itemize}
        \item[I] $\varphi_{\delta \to T}$ is a bi-Lipschitz $C^1$-diffeomorphism for all $0<\delta<T<\infty$,
        \item[II] $\lim_{T \to \infty} \, \DKL{(\varphi_{\delta \to T})_\# p_\delta}{\gaussian} = 0$ for all $\delta > 0$,
        \item[III] $\lim_{\delta \to 0} \, \lim_{T \to \infty} \, \| \ptarget - (\varphi_{\delta\to T}^{-1})_\# \gaussian \|_{L^1} = 0$.
    \end{itemize}
\end{theorem} 

Theorem~\ref{thm:score} shows how regularity of the VP score translates into regularity of the associated transport maps. Assumption~(i) is a uniform-in-space Lipschitz bound on the score, with time-dependent constant $L(t)$, and~(ii) requires that this bound is integrable on any finite time interval away from $t=0$. Under these assumptions, the flow map $\varphi_{\delta\to T}$ is bi-Lipschitz for all $0<\delta<T<\infty$, but the corresponding Lipschitz constants may depend on $\delta$ and $T$ and can potentially blow up as $\delta\rightarrow 0$ or $T\rightarrow\infty$. Nevertheless, III shows that under uniform-in-space score regularity and integrability in time, we can approximate our target distribution arbitrarily well by Gaussian pullbacks induced by such bi-Lipschitz transport maps with respect to $L^1$, as soon as we choose a sufficiently expanded time window $[\delta,T]$. Thus, although working with a cutoff $\delta > 0$ in the transport, the $L^1$-limit is the undiffused target 
$\ptarget$, not the smoothed marginal $p_\delta$.  

The Gaussian smoothing inherent in the VP process regularizes $s_t$, so that boundedness of $\nabla s_t$ in assumption (i) holds for a variety of distributions, including compactly supported and certain log-concave densities, as we show below. Note that even heavy-tailed distributions are not excluded a priori. However, the validity of the assumption depends on how the tails influence the posterior covariance in the score representation, see the score equality further down in~\eqref{eq_stephanovitch_score}.

Statement~I follows from standard flow theory for ODEs with spatially Lipschitz vector fields and part~II follows immediately from~\cite[Lemma~C.4]{Chen2023}. Note that, since the probability flow ODE reproduces the VP-SDE marginals, we have 
$(\varphi_{\delta \to T})_\# p_\delta = p_T$ by definition. Hence, Part~II is essentially the classical KL convergence of the VP diffusion marginals $p_T$ to the standard Gaussian as $T \to \infty$, stated here in transport form to match the pullback viewpoint used in Part~III. Equivalently, this yields convergence of the Gaussian pullback to the early-stopped density $p_\delta$ on the target side
 \begin{equation*}
     \DKL{p_\delta}{(\varphi_{\delta \to T}^{-1})_\#\gaussian} = \DKL{(\varphi_{\delta \to T})_\# p_\delta}{\gaussian} = \DKL{p_T}{\gaussian}
    \xrightarrow[T\to\infty]{} 0.
 \end{equation*}
Although the ingredients are classical, to the best of our knowledge, the combination of Parts~II–III, formulated in terms of bi-Lipschitz VP transport maps and $L^1$ approximation of $\ptarget$ by Gaussian pullbacks, does not appear in this explicit form in the existing literature, imposing score regularity only for positive times away from $t=0$.

We extend these approximation properties by further making an assumption on the uniform-in-time Lipschitz property of the score. 

\begin{corollary}\label{cor:score}
    Assume the setting of Theorem~\ref{thm:score}. In addition to (i) and (ii), suppose that $\ptarget \in C^2_b(\R^n)$, $\ptarget > 0$ and 
    \begin{equation*}
    \sup_{x \in \R^n} \|\nabla_x s_0(x)\| \leq L(0) < \infty \qquad \text{and} \qquad \int_0^T(1+L(t))\, \d t < \infty \text{ for every } 0 < T < \infty.
    \end{equation*}
    Then,
\begin{itemize}
    \item[I] $\varphi_{0 \to T}$ is a bi-Lipschitz $C^1$-diffeomorphism for all $T > 0$.
    \item[II] $\lim_{T \to \infty} \, \DKL{\ptarget}{(\varphi_{0\to T}^{-1})_\# \gaussian}= 0 \quad \text{and} \quad  \lim_{T \to \infty} \, \| \ptarget - (\varphi_{0\to T}^{-1})_\# \gaussian\|_{L^1}= 0.$
\end{itemize}
\end{corollary}

Consequently, these additional assumptions of Corollary~\ref{cor:score} lead to stronger convergence properties in the target space. Note that these assumptions are similar to~\cite{Chen2023}, where non-early-stopped KL convergence is also shown for targets with finite second moment and score regularity on $0 \leq t \leq T$. While~\cite{Chen2023} work in the discretized reverse-time SDE framework, we emphasize a deterministic pullback formulation via the probability flow transport. The assumptions of Theorem~\ref{thm:score} only require the score to be Lipschitz uniformly in space with an integrable bound over a finite time horizon away from $t=0$. Corollary~\ref{cor:score} shows that if, additionally, the score is uniformly Lipschitz on a time window $[0,T]$, the associated flow maps become uniformly bi-Lipschitz on $[0,T]$, i.e. there exist uniform bi-Lipschitz constants $\overline{L}_T, \underline{L}_T$ such that 
\begin{equation*}\label{eq:Lip_const}
    \sup_{0 \leq s < t \leq T} \Lip(\varphi_{s \to t}) \leq \overline{L}_T \quad \text{and} \quad \sup_{0 \leq s < t \leq T} \Lip(\varphi_{s \to t}^{-1}) \leq \underline{L}_T.
\end{equation*}
 This allows us to transfer KL-convergence in latent space to KL-convergence on the target side. In particular, we obtain convergence of the Gaussian pullback to $\ptarget$ instead of an early-stopped density $p_\delta$ via the VP flow. 

The regularity of the score strongly depends on the initial law $\ptarget$. However, it is not fully investigated which target classes fulfill the assumptions in Theorem~\ref{thm:score} or Corollary~\ref{cor:score}. In the following, we identify structural classes of densities for which non-uniform or uniform score conditions are guaranteed.

\begin{lemma}\label{lemma:pH_compact}
Let $\ptarget:\R^n \longrightarrow [0,\infty)$ be a pdf, 
and $s_t(x)$ the score of the VP-SDE~\eqref{eq:VP-SDE_fixedbeta}. Assume $\ptarget$ has compact support, i.e. there exists a compact set $K\subset\R^n$ such that
\begin{equation}
     \supp(\ptarget)\subseteq K. \tag{A1}\label{eq:A1}
\end{equation}
Then, there exists a function ${L:[0,\infty) \rightarrow [0,\infty)}$ such that  
\begin{itemize}
    \item[(i)] $\sup_x \| \nabla s_t(x)\| \leq L(t) < \infty$ for all $t > 0$,
    \item[(ii)] $\int_{\delta}^T (1+L(t))\, \d t < \infty $ for all $0 < \delta < T < \infty$.
\end{itemize}
\end{lemma}

A similar statement also holds for some log-concave densities, as summarized in the following lemma.

\begin{lemma}\label{lemma:score_log_concave}
Let $\ptarget:\R^n \longrightarrow [0,\infty)$ be a pdf, 
and $s_t(x)$ the score of the VP-SDE~\eqref{eq:VP-SDE_fixedbeta}. Let
\begin{equation}
    \ptarget(x) = \frac{1}{Z}\,e^{-V(x)}, \qquad Z>0, \tag{A2}\label{eq:A2}
\end{equation}
where $V\in C^2(\R^n)$ such that
\begin{equation}\label{eq:log_concave}
  \nabla^2 V(x) \succeq 0
  \quad\text{for all }x\in\R^n.
\end{equation} 
Then, there exists a function ${L:[0,\infty) \rightarrow [0,\infty)}$ such that 
\begin{itemize}
    \item[(i)] $\sup_x \| \nabla s_t(x)\| \leq L(t) < \infty$ for all $t > 0$,
    \item[(ii)] $\int_{\delta}^T (1+L(t))\, \d t < \infty $ for all $0 < \delta < T < \infty$.
\end{itemize}
\end{lemma}
The statements of Lemma~\ref{lemma:pH_compact} and Lemma~\ref{lemma:score_log_concave} rely on a representation of the score gradient of the VP flow, which holds for more general target classes than~\eqref{eq:A1} and~\eqref{eq:A2}. In \cite[Proposition~4]{Stephanovitch2025}, such a formula is derived for the VP process $\mathrm{d}X_t = -X_t\,\mathrm{d}t + \sqrt{2}\,\mathrm{d}W_t$, i.e. the choice $\beta(t) \equiv 2$. Our VP process~\eqref{eq:VP-SDE_fixedbeta} coincides with this under the time-change $t \rightarrow t/2$ and the same computation yields  
\begin{equation}\label{eq_stephanovitch_score}
    \nabla s_t(x)
      = -I_n
        + \frac{e^{-t}}{(1-e^{-t})^2}\,\Cov_{p_{t,x}}(Y)
        - \frac{e^{-t}}{1-e^{-t}}\,I_n \quad \text{ for every } t > 0, \ x\in \R^n,
\end{equation}
where $\Cov_{p_{t,x}}(Y)$ denotes the covariance of $Y\sim p_{t,x}$. This implies 
\begin{equation}\label{eq:stephanovitch_score_norm}
      \|\nabla s_t(x)\| \leq 1 + \frac{e^{-t}}{(1-e^{-t})^2}\,\|\Cov_{p_{t,x}}(Y)\| + \frac{e^{-t}}{1-e^{-t}} \quad \text{ for every } t > 0, \ x\in \R^n.
\end{equation}
In our proofs of Lemma~\ref{lemma:pH_compact} and Lemma~\ref{lemma:score_log_concave} we show that $\|\nabla s_t(x)\| < \infty$ using the upper bound in~\eqref{eq:stephanovitch_score_norm}.

A way to recover stronger, uniform-in-time score regularity is to smooth the target by convolution with a Gaussian density. We define the class of densities 
\begin{equation}\label{eq:G}
  \mathcal G := \bigcup_{\Sigma \succ 0} \left\{
        p = \phi_{\Sigma,0} \ast \mu \;:\;
        \mu \text{ is a probability measure with compact support in }\R^n \right\}.
\end{equation}
By construction, every $p \in \mathcal G$ is a smooth, positive pdf on $\R^n$ with finite second moment and all derivatives of $p$ are bounded, in particular, $p \in C_b^\infty(\R^n)$. The next lemma shows that for such Gaussian-smoothed, compactly supported targets one obtains a uniform bound on $\nabla s_t$ on $[0,T]$, including $t=0$.

\begin{lemma}\label{lemma:score_prop_conv}
Let $\Sigma\in\R^{n\times n}$ be symmetric positive definite and let
\begin{equation}
    \ptarget(x) = (\phi_{\Sigma,0} \ast \mu)(x)
    = \int_{\R^n} \phi_{\Sigma,0}(x-y)\,\d\mu(y),
    \label{eq:A3} \tag{A3}
\end{equation}
for some probability measure $\mu$ with compact support $\supp(\mu) \subset B(0,R)$ for a suitable $R>0$, so that $\ptarget\in\mathcal G$.

Let $s_t(x)$ be the score of the VP-SDE~\eqref{eq:VP-SDE_fixedbeta} and initial density $\ptarget$. 
Then, there exists a function ${L:[0,\infty) \rightarrow [0,\infty)}$ such that
\begin{itemize}
  \item[(i)]
    $\sup_{x} \big\|\nabla s_t(x)\big\| \leq L(t) < \infty$ for all $t \geq 0$,
  \item[(ii)]$\int_0^T (1+L(t))\,\d t < \infty$ for all $0 < T < \infty$.
\end{itemize}
\end{lemma}

In particular, for $\ptarget \in \mathcal G$ the VP score satisfies a uniform-in-time Lipschitz bound on every finite interval $[0,T]$, so Gaussian smoothing upgrades the weaker control of Lemma~\ref{lemma:pH_compact} at some time point away from $t=0$ to a global bound that also covers $t=0$ itself. Importantly, this imposes the existence of uniform bi-Lipschitz constants on $[0,T]$ as in~\eqref{eq:Lip_const} of the respective VP flow for all targets of the class~\eqref{eq:A3}. 

Finite Gaussian mixtures provide an important and widely used subclass of $\mathcal G$. Hence Lemma~\ref{lemma:score_prop_conv} applies directly. 

\begin{corollary}\label{cor:score_gauss_mixture}
Let $\ptarget:\R^n \to [0,\infty)$ be a pdf of the form
\begin{equation}
    \ptarget(x) = \sum_{k=1}^K \alpha_k \,\phi_{\Sigma,m_k}(x),
  \qquad \alpha_k>0,\ \sum_{k=1}^K \alpha_k=1, 
    \tag{A4}\label{eq:A4}
\end{equation}
where $\phi_{\Sigma,m_k}$ is the density of $\mathcal N(m_k,\Sigma)$ with a fixed symmetric positive definite $\Sigma\in\R^{n\times n}$ and $m_k \in \R^n$. Let $s_t(x)$ be the score of the VP-SDE~\eqref{eq:VP-SDE_fixedbeta} and initial density $\ptarget$. 
Then, there exists a function ${L:[0,\infty) \rightarrow [0,\infty)}$ such that
\begin{itemize}
    \item[(i)] $\displaystyle \sup_{x\in\R^n} \| \nabla s_t(x)\| \leq L(t) < \infty$ for all $t \geq 0$,
    \item[(ii)] $\displaystyle \int_0^T (1+L(t))\,\d t < \infty$ for all $0 < T < \infty$.
\end{itemize}
\end{corollary}

Collecting Lemmas~\ref{lemma:pH_compact}--\ref{lemma:score_prop_conv} and Corollary~\ref{cor:score_gauss_mixture}, we can now translate these score bounds into approximation guarantees for the VP probability flow. Under any of the structural assumptions~\eqref{eq:A1}--\eqref{eq:A4}, the induced VP transport maps are bi-Lipschitz and their Gaussian pullbacks approximate the target distribution both on the latent and on the data side.

\begin{corollary} \label{cor:approx_prop}
Let $\ptarget$ be a pdf with $\E_{\ptarget}[\|X\|^2] < \infty$, which fulfills~\eqref{eq:A1},~\eqref{eq:A2},~\eqref{eq:A3} or~\eqref{eq:A4}. Then, for every $\varepsilon>0$ there exist $\delta>0$ and $T>\delta$ such that 
\begin{equation*}
    \varphi_\varepsilon:=\varphi_{\delta\to T}
\end{equation*}
is a bi-Lipschitz diffeomorphism, which satisfies
\begin{enumerate}
\item[\textnormal{(i)}] (Forward/latent side) 
\begin{equation*}
    \DKL{(\varphi_\varepsilon)_\# p_\delta}{\gaussian}<\varepsilon,\quad
\big\|(\varphi_\varepsilon)_\# p_\delta-\gaussian\big\|_{L^p(\R^n)}<\varepsilon,\quad
W_p \big((\varphi_\varepsilon)_\# p_\delta,\gaussian\big)<\varepsilon, \ p=1,2,
\end{equation*}
with $p_\delta$ the VP flow marginal at $t=\delta$,
\item[\textnormal{(ii)}] (Backward/target side) 
\begin{equation*}
    \big\|\ptarget-(\varphi_\varepsilon^{-1})_\#\gaussian\big\|_{L^1}<\varepsilon.
\end{equation*}
\end{enumerate}
Additionally, if $\ptarget$ fulfills~\eqref{eq:A3} or~\eqref{eq:A4}, for every $\varepsilon >0$ there exists $T > 0$ such that $\varphi_\varepsilon := \varphi_{0 \to T}$ is a bi-Lipschitz diffeomorphism with  
\begin{equation*}
    \DKL{\ptarget}{(\varphi_\varepsilon^{-1})_\# \gaussian} < \varepsilon.
\end{equation*}
\end{corollary}

Corollary~\ref{cor:approx_prop} makes explicit how the strength of the score regularity assumption controls the expressivity guarantee. Under the weaker assumptions~\eqref{eq:A1}--\eqref{eq:A2}, the VP score is known to be uniformly Lipschitz in space on time intervals $[\delta,T]$ with $\delta>0$, so the bi-Lipschitz constants of $\varphi_{\delta\to T}$ may deteriorate as $\delta \rightarrow 0$ or $T\rightarrow\infty$, and the approximation of $\ptarget$ is obtained via an early-stopped density $p_\delta$ in $L^1$. In contrast, for the Gaussian-convolved classes~\eqref{eq:A3}–-\eqref{eq:A4} we obtain additional uniform-in-time score bounds, which yield VP flows $(\varphi_{0\to T})_{T>0}$ with bi-Lipschitz constants that are uniform on each fixed finite horizon starting at $t=0$ and allow us to upgrade the convergence on the latent side to KL-convergence to $\ptarget$ on the target side. Table~\ref{tab:vp-summary} highlights this metric hierarchy induced by score regularity.

To turn the latter results into a universal approximation statement, we show that $\mathcal G$ is large enough to approximate arbitrary target densities. The following lemma is a standard density result. We include a short proof in Appendix~\ref{sec:appendix_proofs} for completeness.

\begin{lemma}\label{lemma:G_density}
    Let $\ptarget:\R^n \to [0,\infty)$ be a pdf. Then, for every $\varepsilon > 0$ there exists $p \in \mathcal G$ such that 
    \begin{equation*}
        \|\ptarget - p\|_{L^1} < \varepsilon.
    \end{equation*}
\end{lemma}

Combining this density property with the approximation guarantees of Corollary~\ref{cor:approx_prop} for targets in $\mathcal G$ yields the following universal expressivity result for VP transport maps.

\begin{corollary}\label{cor:universal_vp}
Let $\ptarget:\R^n\to[0,\infty)$ be a pdf. Then for every $\varepsilon>0$ there exist $p \in \mathcal G$ and $T>0$ such that, for $(\varphi_{0\to t})_{t\in[0,T]}$ denoting the probability flow~\eqref{eq:VP-flow} with initial law $p$, the following hold:
\begin{itemize}
    \item $\varphi_{0\to T}:\R^n\to\R^n$ is a bi-Lipschitz $C^1$-diffeomorphism,
    \item the associated scores $s_t(x) = \nabla_x \log p_t(x)$ of the VP marginals $(p_t)_{t\in[0,T]}$ are uniformly Lipschitz on $[0,T] \times \R^n$, i.e.,
          \begin{equation*}
              \sup_{t\in[0,T]} \sup_{x\in\R^n} \|\nabla_x s_t(x)\|
              < \infty,
          \end{equation*}
    \item the pullback of the standard Gaussian by $\varphi_\varepsilon := \varphi_{0\to T}$ approximates $\ptarget$ in $L^1$,
          \begin{equation*}
             \big\|\ptarget - (\varphi_\varepsilon^{-1})_\# \gaussian\big\|_{L^1} < \varepsilon. 
          \end{equation*}
\end{itemize}
\end{corollary}

The universal approximation result in Corollary~\ref{cor:universal_vp} is an existence result for VP flows starting from smoothed approximants $p\in \mathcal{G}$, not necessarily the actual data law $\ptarget$. Corollary~\ref{cor:universal_vp} therefore highlights the central role of the class $\mathcal G$. For probability densities $p \in \mathcal G$ we obtain strong, uniform-in-time score bounds and hence bi-Lipschitz VP flows with KL- and $L^1$ approximation guarantees as in Corollary~\ref{cor:approx_prop}. Lemma~\ref{lemma:G_density} shows that $\mathcal G$ is rich enough to approximate any target in $L^1$, and Corollary~\ref{cor:universal_vp} transfers the regularity of $\mathcal G$ to arbitrary targets via such approximations: up to arbitrary $L^1$-accuracy, every $\ptarget$ can, for each $\varepsilon >0$, be realized as a Gaussian pullback of a VP flow with uniformly Lipschitz scores on a finite time horizon. These results show that bi-Lipschitz constraints, which are desirable for stability and invertibility, do not fundamentally limit the expressive power of continuous-time VP flows. 

\section{Connection to Learned Score Functions and Normalizing Flows}\label{sec:connection_learned}

In practice, score-based diffusion models do not have access to the exact VP score $(s_t)_{t \geq 0}$, but to a learned approximation obtained from data. Concretely, a time-dependent score model $s_t^{(m)}$ with $m\in \N$ is trained by minimizing a denoising score matching objective
\begin{equation*}
\min_{(s_t^{(m)})_{t\in[0,T]}} \int_0^T w(t)\,\mathbb{E}_{X_t\sim p_t} \bigl[\|s_t^{(m)}(X_t)-s_t(X_t)\|^2\bigr]\,\d t .
\end{equation*}
for a suitable time weighting function $w(t)$. This naturally raises the question whether the expressivity guarantees for VP probability flows established in Section~\ref{sec:expressivity} persist when the true score is replaced by such a trained model. Using similar assumptions as in previous works, our first result shows that if a sequence of learned scores $(s_t^{(m)})_{m\in\mathbb{N}}$ achieves small score matching error on $[\delta,T]$ in an $L^2$-approximation sense, then the corresponding learned VP flows still generate bi-Lipschitz transport maps whose Gaussian pullbacks approximate the target $\ptarget$ in $L^1$~(see Appendix~\ref{sec:appendix_learned} for the proof).

\begin{theorem}\label{thm:learned_score}
In addition to the assumptions of Theorem~\ref{thm:score}, let $s_t^{(m)}:\R^n \to \R^n$ be a (learned) score for each $m\in\N, \ t\in[\delta,T]$ and define
\begin{equation}\label{eq:learned_flow}
    v_t^{(m)}(x) := -\tfrac12 x - \tfrac12 s_t^{(m)}(x),
\end{equation}
with corresponding ODE flow $\varphi^{(m)}_{s\to t}$.

Assume that for each $m\in\N$ and all $0<\delta<T<\infty$ the following properties hold.
\begin{itemize}
  \item[(i)] The map $(t,x)\mapsto s_t^{(m)}(x)$ is continuous and for ${L^{(m)}:[0,\infty) \rightarrow [0,\infty)}$ we have
  \begin{equation*}
      \sup_{x\in\R^n} \|\nabla_x s_t^{(m)}(x)\| \leq L^{(m)}(t) <\infty
    \quad\text{for all } t>0, \qquad \int_\delta^T \bigl(1+ L^{(m)}(t)\bigr)\,\d t < \infty.
  \end{equation*}
  \item[(ii)] Consider the reverse-time SDE on $[\delta,T]$
  \begin{equation}\label{eq:rsde_learned}
      \d X_t^{(m)}  = \left[-\tfrac12 X_t^{(m)} - s_t^{(m)}(X_t^{(m)})\right]\,\d t + \d W_t, \qquad t\in[\delta,T].
  \end{equation}
  Let $Q^{(m)}$ be its path law on $C([\delta,T];\R^n)$ with terminal law $X_T^{(m)} \sim \gaussian$ and denote its marginal densities by $q_t^{(m)}$.
  
  Let $P$ denote the path law of the corresponding reverse-time SDE with true score $s_t$, marginals $p_t$ and terminal law $X_T \sim p_T$.

  Let $\overline Q^{(m)}$ denote the path law of~\eqref{eq:rsde_learned} with terminal law $\overline X_T^{(m)}\sim p_T$. Assume a Novikov condition
  \begin{equation*}
      \E_{\overline Q^{(m)}}\left[\exp\left(\frac12\int_\delta^T  \big\|s_t^{(m)}(\overline X_t^{(m)}) - s_t(\overline X_t^{(m)})\big\|^2\,\d t\right)\right] < \infty.
  \end{equation*}

  \item[(iii)] Let $\; \mathcal E_{\delta,T}(s^{(m)}):= \int_\delta^T \int_{\R^n} p_t(x)\,\|s_t(x)-s_t^{(m)}(x)\|^2\,\d x\,\d t \xrightarrow[m\to\infty]{} 0.$
  \item[(iv)] For $t\in[\delta,T]$ let $s_t^{(m)}(x) = \nabla_x \log q_t^{(m)}(x)$ for $q_t^{(m)}$-a.e. and $q_t^{(m)} \in C^1$, $q_t^{(m)} > 0$. 
\end{itemize}
Then $\varphi^{(m)}_{\delta\to T}$ is bi-Lipschitz for all $m\in\N$ and, for every
$\varepsilon>0$, there exist $\delta>0$, $T>\delta$ and $m\in\N$ such that
\begin{equation*}
  \big\|\ptarget - ((\varphi^{(m)}_{\delta\to T})^{-1})_\#\gaussian\big\|_{L^1} < \varepsilon.
\end{equation*}
\end{theorem}

\begin{remark}
Assumption~(ii) is a standard Girsanov-Novikov integrability condition, ensuring absolute continuity between the path laws induced by the learned and true reverse-time dynamics. Such argument is classical in the analysis of score-based diffusion models with learned scores, but Novikov's condition need not hold in full generality~\cite{Chen2023b, Chen2023}. In particular, \cite{Chen2023} explicitly note this issue and verify Novikov's condition in their setting. Theorem~\ref{thm:learned_score} combines this classical condition with global score regularity to obtain \emph{deterministic, bi-Lipschitz transport maps} induced by \emph{learned} probability flow ODEs, together with an explicit $L^1$ approximation guarantee. We take the corresponding integrability requirement as an assumption and do not verify it here.

Note that Assumption~(iv) is a restrictive consistency condition implying that the learned probability flow ODE shares the same marginal densities $q_t^{(m)}$ as the learned reverse-time SDE. We use this assumption to conclude that an $L^2$-score approximation error can yield an explicit $L^1$-guarantee for the deterministic ODE sampler, leaving approximate variants of \emph{(iv)} and training considerations to future work.
\end{remark}

%Comment: one can show that fpr p_H heavy tailed, all p_t will also be heavy tailed. one would think that this is a contradictoin to the fact that gaussian through bi-lipschitz results close to p_t and is sub-Gaussian. but this is not contradictory, since we plug in gaussian only so an approximate transport not the true transport. so in the target we might still be able to approximate a heavily tailed pdf even though we only obtain sub-gaussian p_t from pulling back the gaussian. 

From a normalizing flow viewpoint, each VP flow $\varphi^{(m)}_{\delta\to T}$ defines a bi-Lipschitz normalizing flow. To complete this section, we now treat the learned VP transport maps as a teacher map and subsequently use a bi-Lipschitz INN to approximate this transport. In this two-step procedure, Corollary~\ref{cor:nf_learner} shows that, if the student flow matches the teacher transport sufficiently well on $\ptarget$, then its Gaussian pullback remains an $L^1$ approximation of the target distribution.

\begin{corollary}\label{cor:nf_learner}
    Let $\varphi_{\delta \to T}^{(m)}$ be the transport map of the VP flow for fixed $m \in \N$, $T > 0$ as in Theorem~\ref{thm:learned_score} such that 
    \begin{equation*}
        \|\ptarget - ((\varphi_{\delta\to T}^{(m)})^{-1})_\# \gaussian \|_{L^1} < \varepsilon
    \end{equation*}
    and let $\hat\varphi:\R^n \rightarrow \R^n$ be a bi-Lipschitz INN. Suppose 
    \begin{equation*}
        \| \hat\varphi_\# \ptarget - (\varphi_{\delta\to T}^{(m)})_\# \ptarget\|_{L^1} < \eta
    \end{equation*}
    for $\eta > 0$. Then
    \begin{equation*}
        \|\ptarget - (\hat\varphi^{-1})_\# \gaussian\|_{L^1} < \eta + \varepsilon.
    \end{equation*}
\end{corollary}
The proof follows immediately from invariance of the $L^1$-norm under bijective pushforwards and the triangle inequality. 

\begin{remark}
As in the second part of Corollary~\ref{cor:approx_prop}, for target classes that additionally satisfy the assumptions of Theorem~\ref{thm:learned_score} at $t=0$, the statements of Theorem~\ref{thm:learned_score} and Corollary~\ref{cor:nf_learner} can be extended to yield bi-Lipschitz regularity and approximation guarantees for the transports $\varphi^{(m)}_{0\to T}$ for all $m\in\N$ and $T>0$.
\end{remark}

\section{Illustrative Numerical Examples}\label{sec:num_examples}

Our results from Section~\ref{sec:expressivity} aim to establish a conceptual link between the two approach types.
In this section, we illustrate the approximation behaviour of bi-Lipschitz end-to-end normalizing flows and score-based diffusion models on a collection of low-dimensional target distributions in practice. 
The experiments are designed to compare the qualitative approximation behaviour of both approaches by systematically studying how imposed regularity constraints affect learned transport maps and the resulting density approximations. We therefore view the following experiments as qualitative illustrations of the theory rather than as a benchmark comparison.

The comparison in this section serves two complementary purposes for selected target densities corresponding to the structural classes~\eqref{eq:A1}--\eqref{eq:A4}: First, it illustrates how the expressivity guarantees derived for VP probability flows translate into practical approximation behaviour of the SDM itself, but also when replacing the continuous-time diffusion transport by an explicitly parametrized normalizing flow. Second, it highlights how architectural regularity constraints on the end-to-end normalizing flow influence transport geometry and density approximation quality. 

We employ invertible residual networks ~(iResNets) as end-to-end normalizing flows since they combine empirical expressivity with explicit architectural bi-Lipschitz guarantees, matching the assumptions of our theory. For SDMs based on the VP-SDE~\eqref{eq:VP-SDE_fixedbeta}, regularity arises implicitly through the learned score field and its associated probability flow.

\textbf{iResNet.} We train iResNets $\varphi_{\theta,L}:\R^n \longrightarrow \R^n$ with network parameters $\theta \in \Theta$ and Lipschitz constraints depending on $L < 1$. An iResNet is constructed as a composition of invertible residual blocks
\begin{equation*}
    \varphi_{\theta,L}^{(i)} = \Id + f_\theta^{(i)}, \quad \Lip(f_\theta^{(i)}) \leq L < 1,
\end{equation*}
so that the end-to-end map
\begin{equation*}
    \varphi_{\theta,L} = \varphi_{\theta,L}^{(k)} \circ \cdots \circ \varphi_{\theta,L}^{(1)}
\end{equation*}
is invertible and defines a bi-Lipschitz normalizing flow. The inverse $x = (\varphi_{\theta,L}^{(i)})^{-1}(y)$ of each block is computed using the convergent fixed-point iteration
\begin{equation*}
    x_{m+1} = y - f_\theta^{(i)}(x_m).
\end{equation*}
For a network with $k$ residual blocks, the resulting function is bi-Lipschitz, with global Lipschitz bounds 
\begin{equation*}
    \Lip(\varphi_{\theta,L}) \leq (1+L)^k \quad \text{and} \quad \Lip(\varphi_{\theta,L}^{-1}) \leq \frac{1}{(1-L)^k}.
\end{equation*}
We train the iResNet via maximum likelihood estimation, i.e. with the training objective
\begin{equation*}
    \min_\theta\mathcal{L}_{\mathrm{MLE}}(\theta) =
    - \mathbb{E}_{x \sim \ptarget}
    \Bigl[ \log \gaussian \bigl(\varphi_{\theta,L}(x)\bigr) + \log \bigl| \det J_{\varphi_{\theta,L}}(x) \bigr| \Bigr].
\end{equation*}

\textbf{Score-based Diffusion Model.} We train a score network $s_{\theta,t}(x)\approx \nabla_x\log p_t(x)$ which matches the VP process with constant diffusion rate $\beta\equiv 1$, cf.~\eqref{eq:VP-SDE_fixedbeta}--\eqref{eq:a_and_sigma}. Training is performed by denoising score matching: sampling $x\sim \ptarget$, $t\sim\mathcal U(0,T)$ and $z\sim\gaussian$ for $T > 0$, we construct the perturbed inputs
\begin{equation*}
    x_t = a(t)\,x + \sigma(t)\,z \qquad\text{with } a(t), \, \sigma(t)\text{ as in }\eqref{eq:a_and_sigma},
\end{equation*}
and minimize the weighted objective
\begin{equation*}
    \min_\theta \mathcal L_{\mathrm{DSM}}(\theta)=\E_{x,t,z}\Big[w(t)\,\big\|s_{\theta,t}(x_t)+\tfrac{1}{\sigma(t)}z\big\|^2\Big]\qquad \text{with} \quad w(t)=\sigma^2(t).
\end{equation*}
To extract deterministic transport maps from the learned score, we solve the VP probability flow ODE~\eqref{eq:VP-flow} with the learned score, i.e. $\dot x_t = v_{\theta,t}(x_t)$, where $v_{\theta,t}(x)=-\tfrac12 x-\tfrac12 s_{\theta,t}(x)$. The forward transport is defined by
\begin{equation*}
    z_T = \varphi_{\delta\to T}(x_0),
\end{equation*}
obtained by numerically integrating the probability flow ODE from $t=\delta$ to $t=T$ starting at $x_0$, where $\delta>0$ is a small numerical cutoff. Conversely, the inverse transport is obtained by integrating the same ODE backward in time from $z_T\sim \gaussian$ to $t=\delta$.
Numerical integration is carried out with an adaptive Runge-Kutta solver.

We note that a fully fair numerical comparison between iResNets and score-based diffusion models is inherently challenging. An iResNet defines an explicit transport map whose forward and inverse evaluations are obtained in closed form via network evaluation. In contrast, a score-based diffusion model induces an implicit transport given by a time-dependent vector field. The required numerical integration for transport evaluation thus introduces an additional source of approximation error that is not present in the iResNet setting.

\subsection{One-Dimensional Experiments}~\label{subsec:num_1D}

We demonstrate the approximation capabilities of the trained networks by selecting four one-dimensional pdfs:

\begin{itemize}
    \item Triangular density: $\ p_\Delta(x):=(1-|x|)\,\mathbf 1_{[-1,1]}(x)$, $x \in \R$.

    \item Two-interval uniform density: $\ p_{2c}(x)=\mathbf 1_{\left[-1,-\tfrac12\right]\cup\left[\tfrac12,1\right]}(x)$, $x \in \R$.
    \item Cubic Gaussian pullback: Let $g(x) := 3x^3 + x.$
    We define the target density as the pullback of $\gaussian$ under $g$, i.e.
    \begin{equation*}
        p_{\mathrm{cubic}}(x):=
        \gaussian\bigl(g(x)\bigr)\,|g'(x)|=\gaussian(3x^3+x)\,(9x^2+1),
        \qquad x\in\mathbb R.
    \end{equation*}
    \item Gaussian mixture model:
    \begin{equation*}
        p_{\mathrm{gmm}}(x):= \sum_{j=1}^3 \alpha_j\,\phi_{\sigma_j^2,0}(x-\mu_j),
        \qquad x\in\mathbb R,
    \end{equation*}
    with parameters $\mu_j=-3,-1,1$, $\sigma_j=0.2,0.35,0.25$, $\alpha_j=0.2,0.5,0.3$ for $j=1,2,3$.
\end{itemize}
For each $\ptarget \in \{p_\Delta,p_{2c},p_{\mathrm{cubic}},p_{\mathrm{gmm}}\}$ there exists a canonical monotone transport map $g_\ast : \mathbb{R} \to \mathbb{R}$ such that $\gaussian = (g_\ast)_\# \ptarget $. 
Per construction, for $\ptarget = p_\mathrm{cubic}$ we can immediately identify $g_\ast = g$. We aim to approximate the corresponding transport $g_\ast$ for each $\ptarget$ and therefore construct one precomputed data set per target density by drawing i.i.d.\ samples $x^{(i)}\sim \ptarget$ for $i= 1,\hdots,N_\mathrm{train}$, that are kept fixed throughout training. Concretely, we use $N_{\mathrm{train}}=60\,000$ training samples and $N_{\mathrm{val}}=10\,000$ validation samples.

For the iResNet we choose an architecture with $k=5$ residual blocks and a three-layer fully connected network of width $64$ for each $f_\theta^{(i)}$, $i=1,\dots,k$, equipped with an Exponential Linear Unit~(ELU) activation, as recommended in~\cite{Behrmann2019}, after each hidden layer. After each residual block, we apply an ActNorm layer to stabilize training. This results in approximately $43\,K$ trainable parameters. We trained three instances of this architecture corresponding to different Lipschitz bounds $L\in\{0.25,\,0.75,\,0.95\}$. Each linear layer within $f_\theta^{(i)}$ was constrained to have a spectral norm at most $L^{1/3}$, enforced via spectral weight normalization. This guarantees $\Lip(f_\theta^{(i)})\leq L$ and follows the approach of~\cite{Miyato2018} combined with the invertible residual network construction of~\cite{Behrmann2019}. 

For the score-based diffusion model we selected a time horizon of $T=3$ and used a residual MLP architecture with three residual blocks, each consisting of two hidden layers of width $64$ with SiLU activations. Time conditioning is implemented via random Fourier feature embeddings~\cite{Tancik2020}, 
followed by a two-layer fully connected time-embedding network with SiLU activations, whose output conditions each residual block. The resulting score network has approximately $46\,K$ trainable parameters. In contrast to the iResNet architecture, no architectural Lipschitz constraint is imposed on the score network.

All models were trained using the Adam optimizer with comparable learning rates and batch sizes, and trained until convergence for a fixed number of epochs. Additional implementation details are kept identical across experiments to ensure a fair comparison between architectures.

\begin{figure}[!htp]
  \centering
  \begin{subfigure}{\linewidth}
    \centering
    \includegraphics[width=\linewidth]{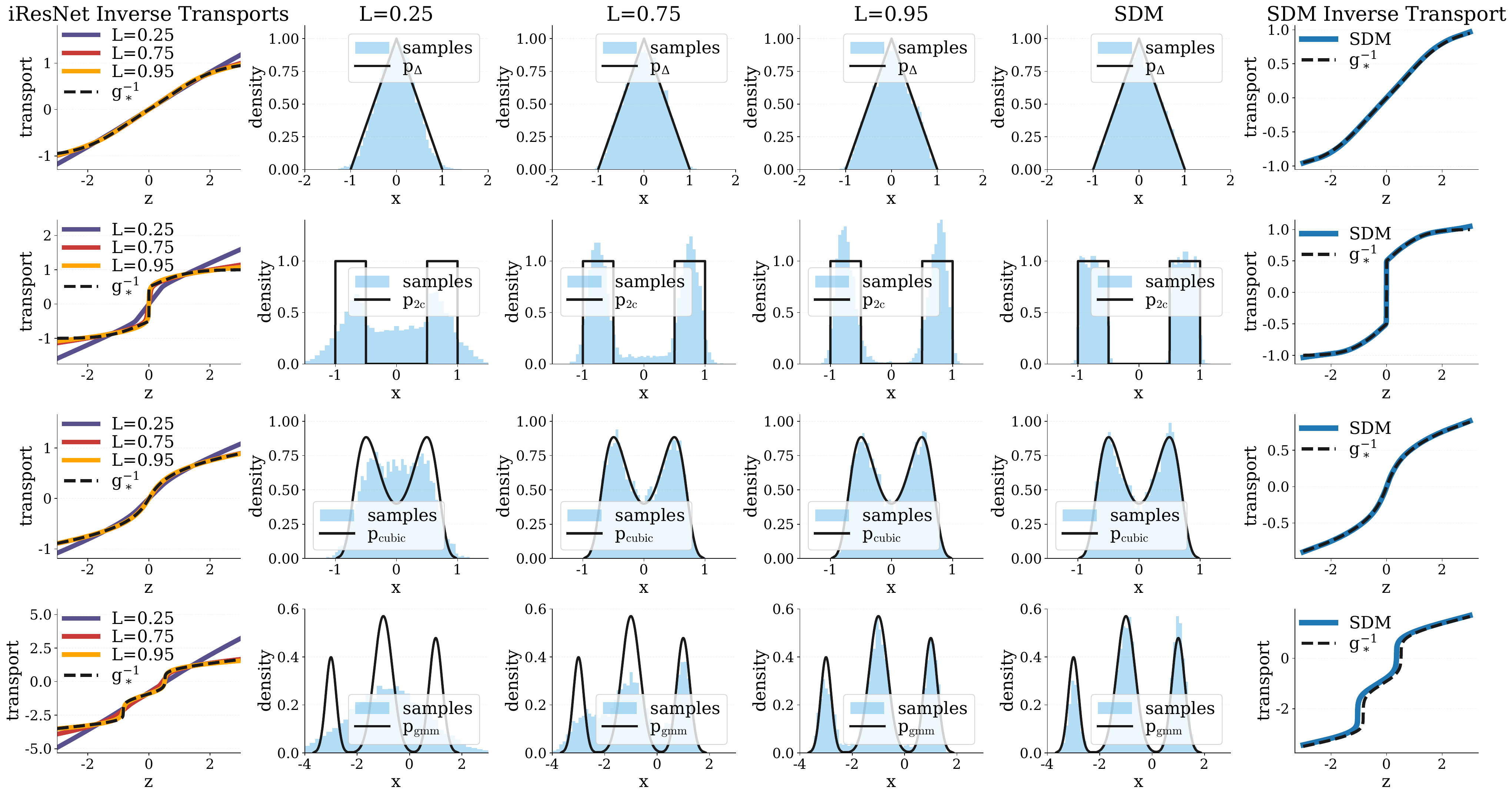}
    \caption{\emph{(left column)} Inverse transports learned by iResNets compared against the ground truth inverse transport $g_\ast^{-1}$ with $\ptarget = (g_\ast^{-1})_\# \gaussian $. \emph{(columns 2--4)} Learned Gaussian pullbacks through the inverse networks. \emph{(right column)} Inverse transport learned by SDM.}
    \label{fig:learned_forward}
  \end{subfigure}

  \vspace{1.3em}

  \begin{subfigure}{\linewidth}
    \centering
    \includegraphics[width=\linewidth]{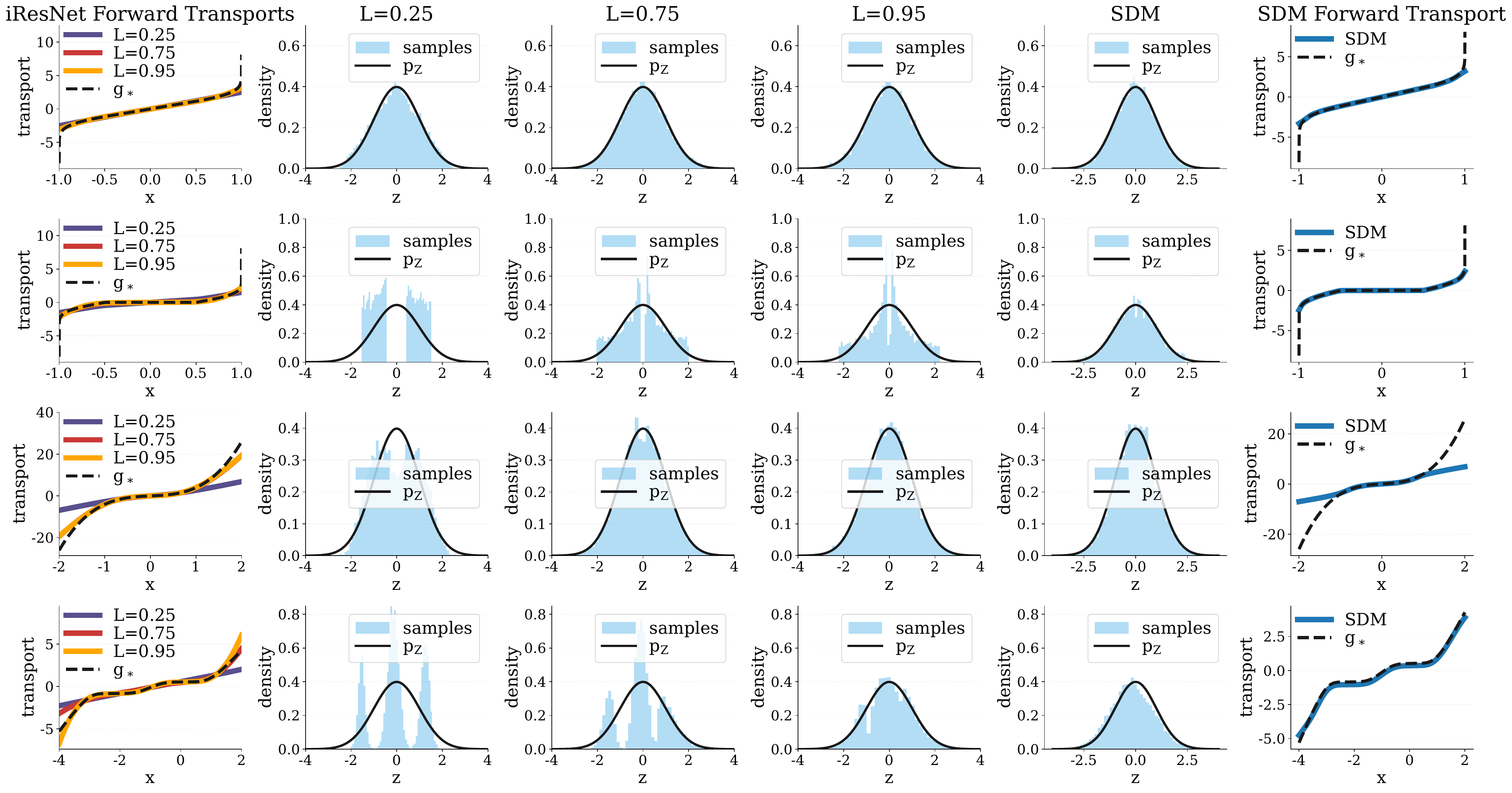}
    \caption{\emph{(left column)} Forward transports learned by iResNets compared against the ground truth transport $g_\ast$ with $\gaussian = (g_\ast)_\# \ptarget$. \emph{(columns 2--4)} Learned pushforwards of $\ptarget$ through the iResNets and the SDM. \emph{(right column)} Forward transport learned by SDM.}
    \label{fig:learned_inverse}
  \end{subfigure}
  \caption{Learned transport maps and densities of iResNets $\varphi_{\theta,L}$ at $L = \{0.25, 0.75, 0.95 \}$ and SDM compared to ground truths. Each row corresponds to one pdf ${\ptarget \in \{p_\Delta, p_{2c}, p_\mathrm{cubic}, p_\mathrm{gmm} \}}$ in the given order.}
  \label{fig:learned_maps}
\end{figure}

Figure~\ref{fig:learned_maps} shows the ground truth targets and corresponding forward and inverse transports together with the learned approximations by each trained model. Probability flow ODEs are solved with an adaptive Runge-Kutta method (RK45), using a small initial time $\delta=10^{-6}$. In these examples we compute histograms by sampling, i.e. by pushing data samples forward or backward through the learned transport maps given by the iResNets and the SDM.

For the iResNet, we observe a clear monotone trend in expressiveness: increasing the block Lipschitz bound from $L=0.25$ to $L=0.95$ (columns 2--4 in Figure~\ref{fig:learned_maps}) systematically  improves the approximation of the ground truth forward and inverse transport maps, and this directly translates into better density approximation. In particular, larger $L$ allows the model to approximate sharp transitions and to redistribute mass onto disconnected or only weakly connected regions, whereas smaller $L$ enforces overly smooth transports, connecting disjoint areas of high ground truth mass. This behaviour is observed most prominently for $p_{2c}, p_{\mathrm{gmm}}$ in Figure~\ref{fig:learned_forward}, rows 2 and 4 and is consistent with the previously mentioned trade-off between regularity and expressivity underlying our theory: while Theorem~\ref{thm:score} and Corollary~\ref{cor:approx_prop} guarantee approximation by bi-Lipschitz transports, the corresponding Lipschitz constants are not required to be uniformly bounded. In contrast, the iResNet imposes a fixed global Lipschitz constraint, so that small values of $L$ restrict the class of admissible transports and lead to underfitting, whereas larger values of $L$ increase the flexibility of the model in terms of e.g. approximating discontinuities or regions of highly separated mass. In practice this constraint on the iResNets results in an expected saturation effect: once the maximum admissible slope of the transport is reached, the slope of the learned transport is clipped to its maximum admissible value in these regions, which prevents matching the non-Lipschitz ground truth transport map in these areas.  
This phenomenon is visible in the first columns of Figure~\ref{fig:learned_forward} and~\ref{fig:learned_inverse}, most clearly in the case of the forward transport of $p_\mathrm{cubic}$ for the tail regions as well as for the inverse transport of $p_{2c}$ near the origin. 

The non-Lipschitz constrained SDM yields accurate density estimates in both data and latent space across all targets and especially captures the discontinuous transport behaviour of $p_{2c}$ most sharply, consistent with its ability to separate the two components of the support. This is in line with Lemma~\ref{lemma:pH_compact}, which ensures score regularity for compactly supported or effectively separated densities away from $t=0$, indicating that $\mathrm{Lip}(\varphi_{\delta \to T}), \mathrm{Lip}(\varphi^{-1}_{\delta \to T})$ will grow as $\delta \to 0$. Since we do not enforce a particular Lipschitz bound of the SDM prior to training, the constant can grow until a satisfactory state of target approximation is reached, which gives more flexibility than the preconstrained iResNets. For the smoother targets, the SDM expectedly achieves accurate approximations, consistent with Lemma~\ref{lemma:score_prop_conv} and Corollary~\ref{cor:score_gauss_mixture}, where we even have uniform-in-time score regularity, resulting in increased transport regularity compared to compactly supported targets. At the same time, the SDM-induced transport behaves conservatively outside the region effectively covered by training data: for $p_\mathrm{cubic}$, the learned forward transport matches the ground truth well on the central interval but flattens outside, indicating limited extrapolation of the learned score field in low-density regions, where the score is only weakly constrained by the training objective, see the plot in Figure~\ref{fig:learned_inverse} at row 4, column 6. This difference arises because, while the gap region of $p_{2c}$ is densely populated by noisy training samples at intermediate diffusion times and therefore induces a strongly learned score, the tail regions of $p_\mathrm{cubic}$ remain sparsely sampled throughout training.

\subsection{Two-Dimensional Experiments}

We next consider four two-dimensional target distributions,
\begin{equation*}
    \ptarget \in \{ p_{\mathrm{rings}},\, p_{\mathrm{squares}},\, p_{\mathrm{moons}},\, p_{\mathrm{conc}} \},
\end{equation*}
illustrated in the left column of Figure~\ref{fig:heat_k=100}. As in the one-dimensional setting, we use fixed data sets of size $N_{\mathrm{train}}=60\,000$ for training and $N_{\mathrm{val}}=10\,000$ for validation.

For the iResNet models, we employ architectures with $k=100$ residual blocks and train five instances corresponding to Lipschitz bounds
$L \in \{0.1,\,0.25,\,0.5,\,0.75,\,0.95\}$. Each residual block uses the same internal architecture as in the 1D experiments, resulting in approximately $865\,\mathrm{K}$ trainable parameters per model. 

For the SDM, we fix the time horizon to $T=3$ and use a residual MLP architecture consisting of nine residual blocks, each implemented as a three-layer MLP with $128$ hidden units. This yields approximately $841\,\mathrm{K}$ trainable parameters. All remaining settings are identical to those used in the one-dimensional experiments.

 For the iResNets, the learned densities are evaluated analytically using the change-of-variables formula
\begin{equation*}
    p_\theta(x) = \gaussian\bigl(\varphi_{\theta,L}(x)\bigr)\,
    \bigl|\det J_{\varphi_{\theta,L}}(x)\bigr|, \qquad x \in \R^2.
\end{equation*}
For the score-based diffusion model, density evaluation is performed via the probability flow ODE. Let $v_{\theta,t}(x) := -\tfrac12 x - \tfrac12 s_{\theta,t}(x)$ as in the beginning of this section. For $x \in \R^2$, the log-density is obtained from
\begin{equation*}
    \log p_\theta(x) = \log p_T(x_T) + \int_{\delta}^{T} \nabla \cdot v_{\theta,t}(x_t)\,\d t,
\end{equation*}
where $x_t$ follows the probability flow ODE, $\delta=10^{-6}$ and $p_T \approx \gaussian$. In two dimensions, the divergence term is computed exactly via automatic differentiation. The resulting ODE system is integrated using an adaptive Runge-Kutta method~(RK45). The density $p_\theta$ is then obtained by exponentiation.

\begin{figure}[t]
    \centering
    \includegraphics[width=\linewidth]{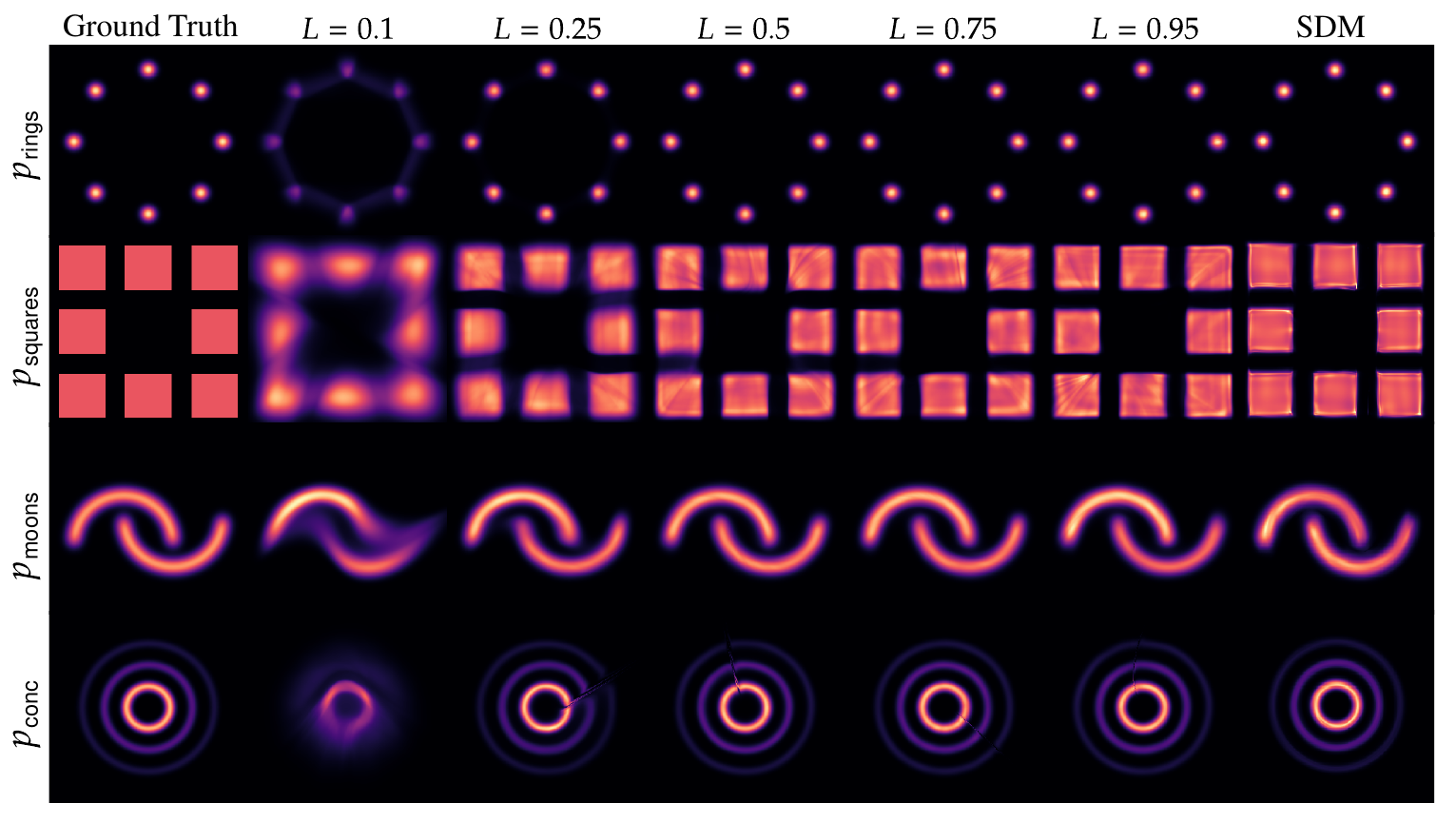}
    \caption{Ground truth target $\ptarget$ and corresponding learned densities $p_\theta$ from iResNets $\varphi_{\theta,L}$ at different Lipschitz constraints $L$ and the SDM~\emph{(columns)} for targets ${\ptarget \in \{ p_{\mathrm{rings}},\, p_{\mathrm{squares}},\, p_{\mathrm{moons}},\, p_{\mathrm{conc}} \}}$~\emph{(rows)}.}
    \label{fig:heat_k=100}
\end{figure}

Figure~\ref{fig:heat_k=100} presents the learned two-dimensional densities for all target distributions. Overall, we observe trends that mirror the one-dimensional experiments. In particular, smaller Lipschitz bounds $L$ enforce smoother transport maps  and consequently lead to overly regularized density estimates (see columns 2--6 of Figure~\ref{fig:heat_k=100}). For target distributions with disconnected or weakly connected support, a sufficiently large Lipschitz constant is essential in order to separate distinct regions of mass. This effect is most pronounced for the square mixture $p_{\mathrm{squares}}$ and the concentric rings distribution $p_{\mathrm{conc}}$, where small values of $L$ lead to artificial connections between components. As in the one-dimensional case, this behaviour is consistent with Lemma~\ref{lemma:pH_compact}, since these targets are either compactly supported or effectively separated and thus require large Lipschitz constants. However, even at large $L$, the iResNet produces an artifact for $p_\mathrm{conc}$, which is a small gap in all concentric rings. In contrast, the ring-shaped Gaussian mixture $p_{\mathrm{rings}}$ is already well approximated for comparatively small values of $L$. This is consistent with Corollary~\ref{cor:score_gauss_mixture}, which ensures uniform-in-time score regularity for finite Gaussian mixture densities and therefore lead to more stable and well-behaved ground truth transport maps. The two-moons distribution $p_{\mathrm{moons}}$ exhibits a slightly different behaviour: here, the best approximation is obtained for an intermediate Lipschitz bound $L=0.5$, while further increasing $L$ does not lead to additional improvements. This suggests that, for this target, a bi-Lipschitz transport of moderate regularity suffices to capture the geometry of the distribution whose Lipschitz constant does not exceed $\Lip(\varphi_{\theta,0.5})$. This observation highlights a trade-off between regularity and generalization: while a sufficiently large Lipschitz constant is necessary to capture nontrivial geometry, overly weak regularization may reduce stability without improving expressiveness.

Since the score-based diffusion model does not impose such preselected Lipschitz constraint on the transport, its visual approximation quality in Figure~\ref{fig:heat_k=100} is similar or improved when comparing with the iResNets: 
While there is no clear visible difference for $p_\mathrm{rings}$ and $p_\mathrm{moons}$, the diffusion model produces a more homogeneous result for the areas of support in $p_\mathrm{squares}$ and, in contrast to the iResNet, is able to produce fully closed concentric rings for $p_\mathrm{conc}$. This qualitative observation is corroborated by the quantitative results shown in Figures~\ref{fig:kl-l1-curves} and Table~\ref{tab:kl-l1}, where the iResNet errors for larger $L$ are mostly in line with the SDM baseline. At the same time, it is important to note a conceptual difference in how these results are obtained: while the SDM induces a canonical and deterministic transport once the score is fixed, the iResNet is trained only at the level of densities and may therefore realize one of many admissible transport maps consistent with the likelihood objective.

Figure~\ref{fig:L(t)_score} provides empirical support for the score regularity results from Section~\ref{sec:expressivity}, which predicts uniform-in-time Lipschitz bounds for sufficiently smooth target classes but allows for singular behaviour near $t=0$ in the presence of discontinuities. In particular, the larger estimated $L(t)$ near $t=0$ for non-smooth targets such as $p_\mathrm{squares}$ is consistent with Lemma~\ref{lemma:pH_compact}, which only guarantees control away from $t=0$ for compactly supported targets, whereas the more moderate behaviour for smoother targets is in line with the uniform-in-time regime of Lemma~\ref{lemma:score_prop_conv} and Corollary~\ref{cor:score_gauss_mixture}. Furthermore, we clearly observe that $L(t)$ converges towards $1$ as $t$ increases for all target distributions, consistent with the fact that $\|\nabla_x^2 \log p_Z(x)\| = 1$.

\begin{figure}[t]
        \centering
        \includegraphics[width=\linewidth]{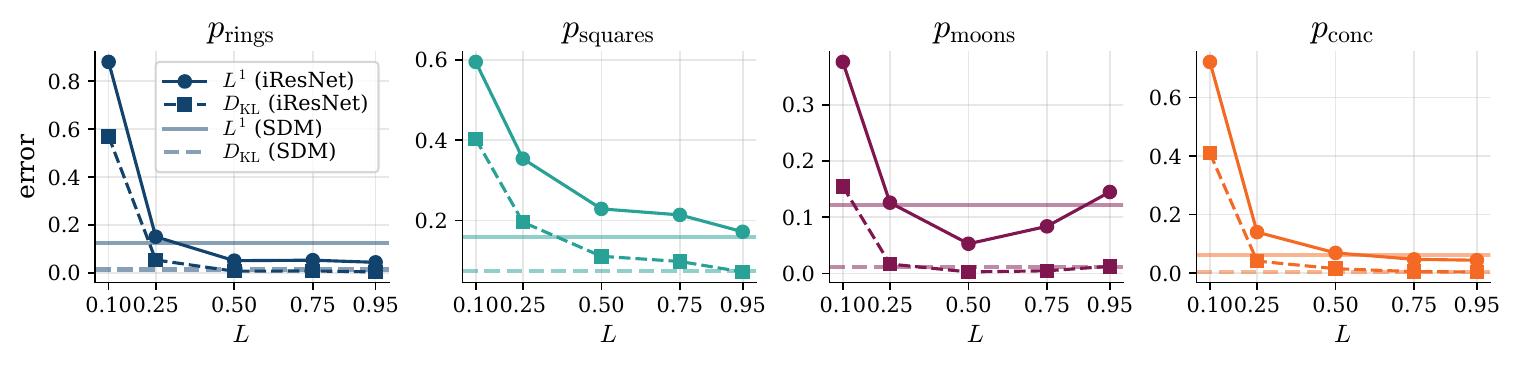}

        \vspace{-3mm}
        \caption{Approximation errors $\| p_\theta - \ptarget \|_{L^1}$~\emph{(dashed lines)} and $\DKL{\ptarget}{p_\theta}$~\emph{(solid lines)} over $L$ compared to the respective constant SDM errors for each target density ${\ptarget \in \{ p_{\mathrm{rings}},\, p_{\mathrm{squares}},\, p_{\mathrm{moons}},\, p_{\mathrm{conc}} \}}$.}
        \label{fig:kl-l1-curves}

\centering
\begin{minipage}[b]{0.37\textwidth}
  \includegraphics[width=5.6cm]{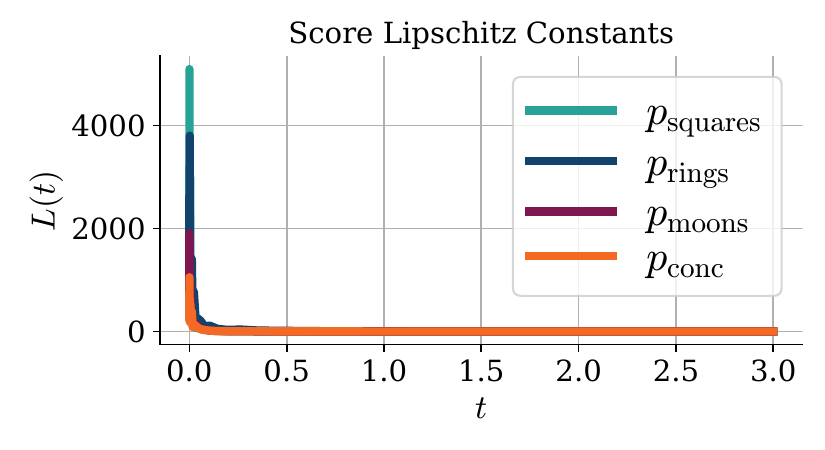}

  \vspace{-4.5mm}
  \includegraphics[width=5.6cm]{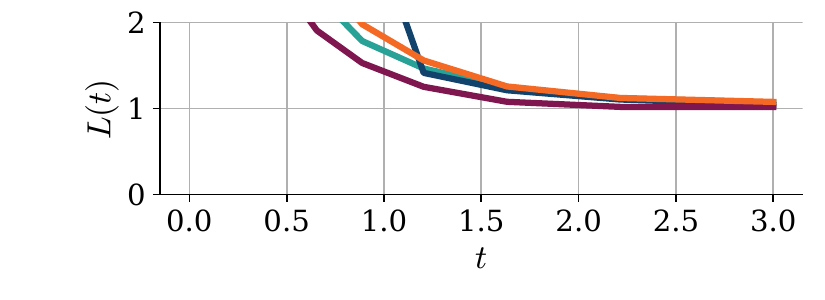}
  
  \vspace{-5mm}
  \captionof{figure}{Lipschitz constants of $s_{\theta,t}(x)$ over time on different  $y$-scales.}
  \label{fig:L(t)_score}
\end{minipage}\hfill
\begin{minipage}[b]{0.61\textwidth}
  \centering
  \footnotesize
  \setlength{\tabcolsep}{1pt}
  \renewcommand{\arraystretch}{1.15}
  \begin{tabular}{@{}llcccccc@{}}
      &  & $L=0.1$ & $L=0.25$ & $L=0.5$ & $L=0.75$ & $L=0.95$ & SDM \\
      \hline\hline
      \multirow{2}{*}{$p_\mathrm{rings}$} & $L^1$
        & 0.881 & 0.151 & 0.051 & 0.053 & \textbf{0.044} & 0.126 \\
      \cline{2-8}
        & $D_{\mathrm{KL}}$
        & 0.570 & 0.054 & 0.007 & 0.008 & \textbf{0.004} & 0.014 \\
      \hline\hline
      \multirow{2}{*}{$p_\mathrm{squares}$} & $L^1$
        & 0.595 & 0.354 & 0.229 & 0.214 & 0.172 & \textbf{0.159} \\
      \cline{2-8}
        & $D_{\mathrm{KL}}$
        & 0.403 & 0.196 & 0.111 & 0.098 & \textbf{0.072} & 0.075 \\
      \hline\hline
      \multirow{2}{*}{$p_\mathrm{moons}$} & $L^1$
        & 0.376 & 0.126 & \textbf{0.053} & 0.084 & 0.145 & 0.122 \\
      \cline{2-8}
        & $D_{\mathrm{KL}}$
        & 0.155 & 0.017 & \textbf{0.003} & 0.005 & 0.013 & 0.012 \\
      \hline\hline
      \multirow{2}{*}{$p_\mathrm{conc}$} & $L^1$
        & 0.722 & 0.140 & 0.069 & 0.047 & \textbf{0.044} & 0.061 \\
      \cline{2-8}
        & $D_{\mathrm{KL}}$
        & 0.410 & 0.041 & 0.015 & 0.005 & 0.004 & \textbf{0.004} \\
      \hline\hline
    \end{tabular}
  \captionof{table}{$L^1$ and $D_{\mathrm{KL}}$ errors for all instances of the iResNet and the SDM for all target densities as in Fig.~\ref{fig:kl-l1-curves}.}
  \label{tab:kl-l1}
\end{minipage}
\end{figure}

\section{Conclusion}

Exploiting the concept of score-based diffusion models enables us to deduce universal distributional approximation properties of bi-Lipschitz normalizing flows. Particularly, we have shown that VP probability flows under uniform-in-space score Lipschitz bounds generate bi-Lipschitz $C^1$-diffeomorphisms whose Gaussian pullbacks are $L^1$-dense among target densities in $L^1$. We identified practically relevant target classes for which the required score regularity holds. The Gaussian-smoothed class $\mathcal G$ plays a key role: it admits additional uniform-in-time score Lipschitz bounds and is rich enough to approximate arbitrary targets in $L^1$. This allows us to transfer bi-Lipschitz expressivity guarantees via VP flows to general probability densities.
Our numerical experiments complement the theory by illustrating how explicit Lipschitz control in iResNets shapes transport geometry and density fit, and how these end-to-end flows compare to VP probability flow transport maps induced by learned scores.

Although the $L^1$-expressivity of deterministic transports is broadly in line with classical volume form results on bounded domains, see~\citeauthor{DacorognaMoser1990}~(\citeyear{DacorognaMoser1990}), our results go further by proving $L^1$-density of Gaussian pullbacks on $\R^n$ within the specific transport family of VP flows arising in diffusion models. Within this class, score Lipschitz bounds translate directly into bi-Lipschitz control of the induced transport maps. Moreover, the derived KL convergence for targets in $\mathcal G$ is strictly stronger than $L^1$. Hence, our results can be interpreted in two complementary directions: On the one hand, VP diffusion regularity provides an existence-based justification for the distributional expressivity of bi-Lipschitz normalizing flows. On the other hand, the bi-Lipschitz deterministic transport viewpoint yields a transparent route to diffusion convergence by viewing sampling as a Gaussian pullback through the VP flow, with score regularity providing bi-Lipschitz control. In this sense, each framework explains and strengthens the other.

A few questions remain open. First, it is not yet clear in which regimes an end-to-end normalizing flow effectively distills the VP probability flow transport. Understanding when such distillation succeeds is an important direction for future work. Second, our analysis is formulated at the level of transport maps rather than specific architectures. A refined, architecture-aware theory could clarify how common design choices (e.g., coupling flows, iResNets) and training objectives translate into the time-dependent bi-Lipschitz regime captured here. Finally, since our results rely on non-uniform Lipschitz bounds, developing approximation and convergence rates that explicitly track these bounds would sharpen the theory and further connect it to the existing quantitative score-based diffusion model literature.

\section{Acknowledgements}
MI acknowledges the support of the Mobility Programme (M-0187) of the Sino-German Center for Research Promotion. CBS acknowledges support from the Royal Society Wolfson Fellowship, the EPSRC advanced career fellowship EP/V029428/1, the EPSRC programme grant EP/V026259/1, the Wellcome Innovator Awards 215733/Z/19/Z and 221633/Z/20/Z, the EPSRC funded ProbAI hub EP/Y028783/1, the European Union Horizon 2020 research and innovation programme under the Marie Skodowska-Curie grant agreement REMODEL.
This research was also supported by the NIHR Cambridge Biomedical Research Centre (NIHR203312). The views expressed are those of the author(s) and not necessarily those of the NIHR or the Department of Health and Social Care.

\appendix
% equations numbered within (appendix) sections
\numberwithin{equation}{section}
\renewcommand{\thesection}{\Alph{section}}
\renewcommand{\theequation}{\thesection.\arabic{equation}}

\section{Proofs of Section~\ref{sec:expressivity}}\label{sec:appendix_proofs}

This appendix collects the proofs of the results stated in Section~\ref{sec:expressivity}. In particular, Appendix~\ref{subsec:appendix_score} proves Theorem~\ref{thm:score}, followed by the proofs of Corollary~\ref{cor:score} and Lemmas~\ref{lemma:pH_compact}–\ref{lemma:G_density} as well as Corollary~\ref{cor:universal_vp}.

\subsection{Proof of Theorem~\ref{thm:score}}\label{subsec:appendix_score}

\begin{proof}
    Property (i) immediately implies
    \begin{equation}\label{eq:s_lip}
        \| s_t(x)-s_t(y)\| \leq L(t) \|x-y\| \quad \text{for all } x,y \in \R^n, \ t > 0,
    \end{equation}
    i.e. $s_t(x)$ is Lipschitz with constant $L(t)$ for each $t > 0$. 
    
    We will sequentially prove I--III.

    I: \emph{Existence and Uniqueness.} According to~\eqref{eq:s_lip}, $v_t$ as in~\eqref{eq:VP-flow} is Lipschitz for all $t > 0$. Furthermore, $v_t(x)$ is continuous on $[\delta,T] \times \R^n$ (see Lemma~\ref{lem:regularity_pt} for a detailed proof), so there exists a unique solution to~\eqref{eq:VP-flow}, i.e., a well-defined flow $\varphi_{\delta \to t}$ for all $t \in [\delta,T]$. 
    
    \emph{Lipschitz-Property.}
    Let $X_t=\varphi_{\delta\to t}(X_\delta)$ and $Y_t=\varphi_{\delta\to t}(Y_\delta)$, and set $W_t:=X_t-Y_t$.
    Using that $v_t$ is Lipschitz with constant $\tfrac12(1+L(t))$, we obtain for all $t\in[\delta,T]$
    \begin{equation*}
        \frac{\d}{\d t}\|W_t\|^2 =2\langle W_t, v_t(X_t)-v_t(Y_t)\rangle\leq 2\|W_t\|\,\|v_t(X_t)-v_t(Y_t)\|\leq (1+L(t))\|W_t\|^2 .
    \end{equation*}
    By Grönwall’s lemma,
    \begin{equation*}
        \|W_t\|^2\leq \exp\left(\int_\delta^t (1+L(\tau))\,\d\tau\right)\|W_\delta\|^2,
    \end{equation*}
    implying
    \begin{equation*}
        \|W_t\|\leq \exp\left(\frac12\int_\delta^t (1+L(\tau))\,\d\tau\right)\|W_\delta\|,
    \end{equation*}
    hence 
    \begin{equation*}
        \Lip(\varphi_{\delta\to T})\leq \exp\left(\frac12\int_\delta^T (1+L(\tau))\,\d\tau\right).
    \end{equation*}
    The same argumentation applies to $\varphi_{T\to t}$ 
    when solving~\eqref{eq:VP-flow} with $-v$: Since the initial value problems with vector fields $v$ (forward) and $-v$ (backward) are well posed on $[\delta,T]$, for the backward direction we also have 
    \begin{equation*}
        \Lip(\varphi_{T \to \delta}) \leq\exp\left(\frac{1}{2}\int_\delta^T(1+ L(\tau)) \, \d \tau\right).
    \end{equation*}
    The composition of the forward flow and the backward flow yields $\varphi_{t\to \delta}\circ \varphi_{\delta\to t}=\mathrm{id}$ and
    $\varphi_{\delta\to t}\circ \varphi_{t\to \delta}=\mathrm{id}$.
    In particular, $\varphi_{\delta\to t}$ is invertible and $(\varphi_{\delta\to t})^{-1}=\varphi_{t\to \delta}$.
    This shows that $\varphi_{\delta \to T}$ is bi-Lipschitz. 

    \emph{$C^1$ and diffeomorphism.}
    Lemma~\ref{lem:regularity_pt} implies that for each $t>0$ the map $x\mapsto v_t(x)$ is $C^1$ with $ \|\nabla_x v_t(\cdot)\|\leq \tfrac12\bigl(1+L(t)\bigr)$ by assumption (i).
    Together with the continuity of $(t,x)\mapsto v_t(x)$ on $[\delta,T]\times\R^n$, according to~\cite[Theorem~9.5]{Amann1990}, for each fixed $t\in[\delta,T]$, the solution map $x \mapsto \varphi_{\delta\to t}(x)$ is $C^1$. This applies to $\varphi_{t \to \delta}$ by the same arguments. Therefore, $\varphi_{\delta\to t}, \ \varphi_{t \to \delta}$ are $C^1$-diffeomorphisms.
    
    II: Follows from~\cite[Lemma~C.4]{Chen2023}, where 
    \begin{equation*}
        \DKL{p_T}{\gaussian} \leq e^{-T}(n+M_2) 
    \end{equation*}
    is shown for $T$ sufficiently large and $M_2 := \E_{\ptarget}[\|X\|^2] < \infty$ by our assumption. This shows that
    \begin{equation*}
        \DKL{(\varphi_{\delta \to T})_\# p_\delta}{\gaussian} = \DKL{p_T}{\gaussian} \xrightarrow[T\to \infty]{} 0,
    \end{equation*}
    which yields II. 
    
    III: By the triangle inequality,
    \begin{equation*}
        \|\ptarget-(\varphi_{\delta\to T}^{-1})_\#\gaussian\|_{L^1} \leq \|\ptarget-p_\delta\|_{L^1} + \|p_\delta-(\varphi_{\delta\to T}^{-1})_\#\gaussian\|_{L^1}.
    \end{equation*}
    
    Since $\varphi_{T\to\delta}=(\varphi_{\delta\to T})^{-1}$ and thus $p_\delta=(\varphi_{\delta\to T}^{-1})_\#p_T$, and total variation is invariant under bijective pushforwards,
    \begin{equation*}
        \|p_\delta-(\varphi_{\delta\to T}^{-1})_\#\gaussian\|_{L^1} =\|(\varphi_{\delta\to T}^{-1})_\#p_T-(\varphi_{\delta\to T}^{-1})_\#\gaussian\|_{L^1} =\|p_T-\gaussian\|_{L^1}.
    \end{equation*}
    
    By Pinsker’s inequality and Part II,
    \begin{equation*}
        \|p_T-\gaussian\|_{L^1}\leq \sqrt{2\,\DKL{p_T}{\gaussian}} \xrightarrow[T\to\infty]{} 0.
    \end{equation*}
    
    Moreover, $t \mapsto p_t(x)$ is strongly continuous at $t=0$ on $L^1$, which can be seen using the explicit formula $p_t=(a(t)^{-n}\ptarget(\cdot/a(t)))\ast \phi_{\sigma^2(t),0}$ and the facts that scaling and Gaussian convolution are $L^1$-continuous at the identity, i.e. $$\| D_{a(t)}\ptarget \ast \phi_{\sigma^2(t),0}-\ptarget\|_{L^1} \xrightarrow{t\to 0} 0,$$  
    where $D_a p := a^{-n} p(\cdot/a)$~(see Lemma~\ref{lemma:L1_conv_pt} for a detailed proof). Therefore, $\|p_\delta-\ptarget\|_{L^1}\xrightarrow[\delta\rightarrow 0]{}0$.
    Combining the two limits yields
    \begin{equation*}
        \lim_{\delta \to 0} \, \lim_{T \to \infty} \, \|\ptarget-(\varphi_{\delta\to T}^{-1})_\#\gaussian\|_{L^1} = 0.
    \end{equation*}
\end{proof}

\subsection{Proof of Corollary~\ref{cor:score}}

\begin{proof}
    If both additional assumptions hold, applying the same arguments as in the proof of Theorem~\ref{thm:score} shows that $\varphi_{0 \to T}$ is also a bi-Lipschitz $C^1$-diffeomorphism with $(\varphi_{0\to T})_\# \ptarget = p_T$. Therefore,
    \begin{equation*}
        \DKL{\ptarget}{(\varphi_{0\to T}^{-1})_\# \gaussian} = \DKL{(\varphi_{0\to T})_\#\ptarget}{\gaussian} = \DKL{p_T}{\gaussian} \xrightarrow{T \to \infty} 0,
    \end{equation*}
    by Theorem~\ref{thm:score}, II. 
    
    The $L^1$-convergence again follows by the same arguments as in the proof of Theorem~\ref{thm:score}, III, using Lemma~\ref{lem:regularity_pt} extended on $[0,T]\times\R^n$.
\end{proof}

\subsection{Proof of Lemma~\ref{lemma:pH_compact}}

\begin{proof}
Following~\eqref{eq:stephanovitch_score_norm}, we first obtain a bound on $\|\Cov_{p_{t,x}}(Y)\|$ that is uniform in $x$. Let $X_0\sim\ptarget$ and $X_t$ be the VP-SDE forward process.
Then by~\eqref{eq:p_t,x}
\begin{equation*}
     p_{t,x}(y) 
  \propto \ptarget(y)\,\phi_{\sigma^2(t),a(t)y}(x).
\end{equation*}
Thus $\ptarget(y)=0$ implies $p_{t,x}(y)=0$ and therefore
\begin{equation*}
    \supp(p_{t,x}) \subset \supp(\ptarget) \subset K \quad\text{for all }x\in\R^n,\ t > 0.
\end{equation*}
Let $R := \sup\{\|y\| : y\in K\} <\infty$. For $Y\sim p_{t,x}$ we then have $\|Y\|\leq R$. For any
$\theta\in\R^n$ with $\|\theta\|=1$,
\begin{equation*}
    \Var_{p_{t,x}}(\langle\theta,Y\rangle)
  \leq \E_{p_{t,x}}[\langle\theta,Y\rangle^2]
  \leq \E_{p_{t,x}}[\|Y\|^2]
  \leq R^2.
\end{equation*}
Taking the supremum over all unit $\theta$ shows
\begin{equation*}
    \|\Cov_{p_{t,x}}(Y)\| = \sup_{\|\theta\|=1} \Var_{p_{t,x}}(\langle\theta,Y\rangle) \leq R^2,
\end{equation*}
and this bound is independent of $x$ and $t$. Plugging this into \eqref{eq:stephanovitch_score_norm} yields
\begin{equation*}
    \|\nabla s_t(x)\| \leq
  1 + \frac{e^{-t}}{(1-e^{-t})^2}\,R^2 + \frac{e^{-t}}{1-e^{-t}} =: L(t),
\end{equation*}
so (i) holds. Finally, for any fixed $0<\delta<T<\infty$, the function
\begin{equation*}
    t\mapsto L(t)
  = 1 + \frac{e^{-t}}{(1-e^{-t})^2}\,R^2  + \frac{e^{-t}}{1-e^{-t}}
\end{equation*}
is continuous and finite on $[\delta,T]$, hence bounded. Therefore
\begin{equation*}
    \int_\delta^T (1+L(t))\, \d t <\infty.
\end{equation*}
\end{proof}

\subsection{Proof of Lemma~\ref{lemma:score_log_concave}}

\begin{proof}
Again, we aim to obtain a bound on $\|\Cov_{p_{t,x}}(Y)\|$ for each $t > 0$ that is uniform in $x$. The posterior density of $X_0 \sim \ptarget$ given $X_t=x$ is by~\eqref{eq:p_t,x}
\begin{equation*}
     p_{t,x}(y) \propto \ptarget(y)\,\exp\big(-c_t\|y\|^2 + \langle u_t(x),y\rangle\big),
\end{equation*}
where
\begin{equation*}
    c_t := \frac{a^2(t)}{2\sigma^2(t)}
       = \frac{e^{-t}}{2(1-e^{-t})} > 0,
  \qquad
  u_t(x) := \frac{a(t)}{\sigma^2(t)}x
          = \frac{e^{-t/2}}{1-e^{-t}}\,x.
\end{equation*}
Thus, $p_{t,x}$ writes as
\begin{equation*}
    p_{t,x}(y)
  = \frac{1}{Z_{t,x}} \exp\big(-U_{t,x}(y)\big),
\end{equation*}
with
\begin{equation*}
    U_{t,x}(y) := V(y) + c_t\|y\|^2 - \langle u_t(x),y\rangle,
  \qquad Z_{t,x}>0.
\end{equation*}
Differentiating twice in $y$ yields
\begin{equation*}
     \nabla_y^2 U_{t,x}(y)
  = \nabla^2 V(y) + 2c_t I_n
  \succeq 2c_t\,I_n 
  \quad\text{for all }y\in\R^n,
\end{equation*}
by~\eqref{eq:log_concave}. Hence, for each fixed $t>0$ and $x\in\R^n$, the posterior $p_{t,x}$ is strongly log-concave.

By the Brascamp-Lieb inequality, if a pdf $q(y)\propto e^{-U(y)}$ satisfies
$\nabla^2 U(y)\succeq m I_n$ for all $y$, then
\begin{equation*}
    \Var_q(\langle\theta,Y\rangle) \leq\frac{1}{m}\|\theta\|^2
  \quad\text{for all }\theta\in\R^n.
\end{equation*}
Applying this to $q=p_{t,x}$ with $m=2c_t$ yields
\begin{equation*}
    \Var_{p_{t,x}}(\langle\theta,Y\rangle)
  \leq \frac{1}{2c_t}\,\|\theta\|^2
  \quad\text{for all }\theta\in\R^n.
\end{equation*}
Equivalently,
\begin{equation*}
    \|\Cov_{p_{t,x}}(Y)\| = \sup_{\|\theta\| = 1}\Var_{p_{t,x}}( \langle \theta, Y \rangle) \leq \frac{1}{2c_t},
\end{equation*}
and this bound is independent of $x$.

Substituting into~\eqref{eq:stephanovitch_score_norm}, we obtain
\begin{equation*}
     \|\nabla s_t(x)\| \leq 1 + \frac{e^{-t}}{(1-e^{-t})^2}\,\frac{1}{2c_t} + \frac{e^{-t}}{1-e^{-t}} = \frac{2}{1-e^{-t}}=: L(t), %1 + \frac{e^{-2t}}{(1-e^{-2t})^2}\,\frac{1}{\kappa+2c_t} + \frac{e^{-2t}}{1-e^{-2t}} =: L(t),
\end{equation*}
for all $x\in\R^n$. This proves (i).

For (ii), fix $0<\delta<T<\infty$. The functions $a(t)$, $\sigma^2(t)$
and hence $c_t$ are smooth in $t$, so $L(t)$ is continuous on
$[\delta,T]$. In particular, $L(t)$ is bounded on $[\delta,T]$, and
therefore
\begin{equation*}
    \int_\delta^T (1+L(t))\, \d t < \infty.
\end{equation*}
\end{proof}

\subsection{Proof of Lemma~\ref{lemma:score_prop_conv}}

\begin{proof}
Fix $t \geq 0$. Let $Y\sim\mu$ and, let $X_t\mid Y=y\sim\mathcal N(a(t)y,S_t)$
with $S_t=a(t)^2\Sigma+\sigma(t)^2I_n$. Then
\begin{equation*}
    p_t(x)=\int_{\mathbb R^n}\phi_{S_t,a(t)y}(x)\,\mu(\d y).
\end{equation*}
Note that $p_t(x)>0$ for all $x$, since $\phi_{S_t,a(t)y}(x)>0$ and $\mu$ is a probability measure.
For $y\in\supp(\mu) \subset B_R(0)$ define $h_y(x):=\log\phi_{S_t,a(t)y}(x)$, so that
\begin{equation*}
    \nabla_x h_y(x)=-S_t^{-1}(x-a(t)y),\qquad \nabla_x^2 h_y(x)=-S_t^{-1}.
\end{equation*}
Define
\begin{equation*}
    \mu_{t,x}(\d y) := \frac{\phi_{S_t,a(t)y}(x)}{p_t(x)}\,\mu(\d y) = \frac{e^{h_y(x)}}{p_t(x)}\mu(\d y).
\end{equation*}
Then $\mu_{t,x}$ is a probability measure since $\int \phi_{S_t,a(t)y}(x)\,\mu(\d y)=p_t(x)$. By the Leibniz rule (justified since $\mu$ is compactly supported and the derivatives of $\phi_{S_t,a(t)y}$ are bounded uniformly in $y$ on $\supp(\mu)$), we obtain
\begin{align*}
    \nabla_x^2\log p_t(x) &= \frac{\nabla^2p_t(x)}{p_t(x)}-\frac{\nabla p_t(x)(\nabla p_t(x))^\top}{p^2_t(x)} \\
    &= \E_{\mu_{t,x}}[\nabla^2 h_Y(x)]+ \E_{\mu_{t,x}}[\nabla h_Y(x) \nabla h_Y(x)^\top]-\E_{\mu_{t,x}}[\nabla h_Y(x)](\E_{\mu_{t,x}}[\nabla h_Y(x)])^\top \\
    &= \E_{\mu_{t,x}}[\nabla^2 h_Y(x)] + \Cov_{\mu_{t,x}}[\nabla_x h_Y(x)] \\
    &= -S_t^{-1} + \Cov_{\mu_{t,x}}[-S_t^{-1}x + S_t^{-1} a(t)Y] \\
    &= -S_t^{-1} + \Cov_{\mu_{t,x}}[S_t^{-1}a(t)Y]
\end{align*}
Hence
\begin{equation*}
    \|\nabla_x^2\log p_t(x)\|\leq \|S_t^{-1}\|+\big\|\operatorname{Cov}_{\mu_{t,x}}(S_t^{-1}a(t)Y)\big\| \leq \|S_t^{-1}\|+\mathbb E_{\mu_{t,x}}\big[\|S_t^{-1}a(t)Y\|^2\big].
\end{equation*}
Since $\Sigma\succ0$ and $t \in [0,T]$, the eigenvalues of $S_t = a(t)^2\Sigma + \sigma(t)^2 I_n$ are uniformly bounded away from $0$ and
$\infty$: there exist $0<\underline\lambda_T \leq \overline\lambda_T < \infty$ such that
\begin{equation*}
  \underline\lambda_T I_n \preceq S_t \preceq \overline\lambda_T I_n,
  \qquad t\in[0,T].
\end{equation*}
In particular,
\begin{equation*}
  \|S_t^{-1}\| \leq \underline\lambda_T^{-1}
  \qquad\text{for all } t\in[0,T].
\end{equation*}
Since $\supp(\mu)\subset B(0,R)$, $a(t)\le1$, and for $t\in[0,T]$ it follows that $\|S_t^{-1}a(t)Y\|\leq \underline\lambda_T^{-1}R$ and therefore
\begin{equation*}
    \sup_{x\in\mathbb R^n}\|\nabla_x^2\log p_t(x)\|
\leq \underline\lambda_T^{-1}+ \underline\lambda_T^{-2}R^2=:M_T<\infty,\qquad t\in[0,T].
\end{equation*}
Since $\nabla s_t=\nabla^2\log p_t$, this proves (i) with $L(t):=M_T$ on $[0,T]$, and (ii) follows immediately.
\end{proof}

\subsection{Proof of Corollary~\ref{cor:score_gauss_mixture}}
\begin{proof}
Define a probability measure $\mu$ on $\R^n$ by
\begin{equation*}
    \mu := \sum_{k=1}^K \alpha_k \,\delta_{m_k},
\end{equation*}
where $\delta_{m_k}$ is the Dirac measure at $m_k$. Then $\supp(\mu) = \{m_1,\dots,m_K\}$ is finite, hence compact.

For $x\in\R^n$ we have
\begin{equation*}
    (\phi_{\Sigma,0} \ast \mu)(x)
    = \int_{\R^n} \phi_{\Sigma,0}(x-y)\,\d\mu(y)
    = \sum_{k=1}^K \alpha_k \phi_{\Sigma,0}(x-m_k)
    = \sum_{k=1}^K \alpha_k \phi_{\Sigma,m_k}(x)
    = \ptarget(x),
\end{equation*}
so $\ptarget$ is of the form \eqref{eq:A3} with this $\Sigma$ and $\mu$,
and thus $\ptarget \in \mathcal{G}$. Therefore Lemma~\ref{lemma:score_prop_conv} applies and yields the existence of a function $L(t)$ such that (i) and (ii) hold. 
\end{proof}

\subsection{Proof of Corollary~\ref{cor:approx_prop}}

\begin{proof}
(i) \emph{(Forward/latent side)}.
Fix $\delta>0$ and write $p_T=(\varphi_{\delta\to T})_\#p_\delta$. 
By the proof of Theorem~\ref{thm:score}~(II),
\begin{equation*}
    \DKL{p_T}{\gaussian}%\leq e^{-(T-\delta)}\DKL{p_\delta}{\gaussian}\xrightarrow[T\to\infty]{}0.
    \leq e^{-T}(n + M_2) \xrightarrow[T\to\infty]{}0
\end{equation*}
with $M_2 := \E_{\ptarget}[\|X\|^2]$. Hence we can choose $T$ large, so that $\DKL{p_T}{\gaussian}<\min\{\varepsilon,\varepsilon^2/2\}$.
Then Pinsker gives the $L^1$ bound
\begin{equation*}
    \|p_T-\gaussian\|_{L^1}\leq\sqrt{2\,\DKL{p_T}{\gaussian}}<\varepsilon.
\end{equation*}

For $L^2$, use the uniform bound $\|p_T\|_\infty\leq C_\delta:=(2\pi\sigma^2(\delta))^{-n/2}$ to get, via Hölder,
\begin{equation*}
    \|p_T-\gaussian\|_{L^2}^2
\leq \|p_T-\gaussian\|_{L^1}\,\|p_T-\gaussian\|_\infty
\leq (C_\delta+\|\gaussian\|_\infty)\, \|p_T-\gaussian\|_{L^1}.
\end{equation*}
Since $C_\delta+\|\gaussian\|_\infty<\infty$ for fixed $\delta$, enlarging $T$ also yields $\|p_T-\gaussian\|_{L^2}<\varepsilon$. 

For the Wasserstein distance, Talagrand’s inequality for the standard Gaussian yields
\begin{equation*}
    W_2^2(p_T,\gaussian)\leq 2\,\DKL{p_T}{\gaussian}.
\end{equation*}
Hence, by choosing $T$ sufficiently large such that
\begin{equation*}
    \DKL{p_T}{\gaussian}<\varepsilon^2/2,
\end{equation*}
we obtain
\begin{equation*}
   W_2(p_T,\gaussian)<\varepsilon. 
\end{equation*}
Since $W_1\leq W_2$, the same bound also holds for $W_1(p_T,\gaussian)$. 
This concludes (i) because $(\varphi_\varepsilon)_\#p_\delta=p_T$ with our choice $\varphi_\varepsilon=\varphi_{\delta\to T}$.

(ii) \emph{(Backward/target side)}. Follows directly from Lemma~\ref{lemma:pH_compact},~\ref{lemma:score_log_concave},~\ref{lemma:score_prop_conv} and Corollary~\ref{cor:score_gauss_mixture} together with Theorem~\ref{thm:score}, III.

The additional property for the cases of~\ref{eq:A3} and~\ref{eq:A4} directly follows from Corollary~\ref{cor:score} and the characteristics of the class $\mathcal{G}$ which imply $\ptarget \in C^2_b(\R^n) \, \ptarget > 0$ for any $\ptarget \in \mathcal{G}$.
\end{proof}

\subsection{Proof of Lemma~\ref{lemma:G_density}}

\begin{proof}
Since $\ptarget \in L^1(\R^n)$ with $\int_{\R^n} \ptarget = 1$, its mass outside large balls can be made arbitrarily small. Fix $\varepsilon>0$ and choose $R>0$ such that
\begin{equation*}
    \int_{B_R^c(0)}\ptarget(x)\,\d x < \varepsilon/4.
\end{equation*}
Let $m_R:=\int_{B_R(0)}\ptarget$ and define the truncated/renormalized density
\begin{equation*}
    \ptarget^R(x):=\frac{1}{m_R}\ptarget(x)\mathbf 1_{B_R(0)}(x).
\end{equation*}
Then $\ptarget^R$ is a compactly supported pdf and
\begin{equation*}
    \|\ptarget-\ptarget^R\|_{L^1}
    = \int_{B_R(0)}\ptarget\Bigl|1-\frac1{m_R}\Bigr| + \int_{B_R^c(0)}\ptarget
    = (1-m_R)+(1-m_R)
    =2\int_{B_R^c(0)}\ptarget
    < \varepsilon/2.
\end{equation*}
Let $\phi_\sigma:=\phi_{\sigma^2 I,0}$ and set $p_\sigma:=\ptarget^R\ast\phi_\sigma$.
Since $\ptarget^R \in L^1$, we have 
\begin{equation*}
    \|p_\sigma-\ptarget^R\|_{L^1}\to0 \quad \text{as}  \ \sigma\to0
\end{equation*}
(see the proof of Lemma~\ref{lemma:L1_conv_pt}), so choose $\sigma>0$ such that $\|p_\sigma-\ptarget^R\|_{L^1}<\varepsilon/2$.

Finally, if $\mu_R(\d y):=\ptarget^R(y)\,\d y$, then $\mu_R$ is a probability measure supported in $B_R(0)$ and
\begin{equation*}
    p_\sigma(x)=\int_{\mathbb R^n}\phi_{\sigma^2 I,0}(x-y)\,\mu_R(\d y)\in\mathcal G.
\end{equation*}
By the triangle inequality, $\|\ptarget-p_\sigma\|_{L^1}<\varepsilon$.
\end{proof}

\subsection{Proof of Corollary~\ref{cor:universal_vp}}

\begin{proof}
Fix $\varepsilon>0$. 

\emph{Step 1.}
By Lemma~\ref{lemma:G_density}, there exists $p\in\mathcal G$ such that
\begin{equation*}
  \|\ptarget - p\|_{L^1} < \frac{\varepsilon}{2}.
\end{equation*}

\emph{Step 2.} By Lemma~\ref{lemma:score_prop_conv}, for the VP-SDE flow with initial pdf $p$, for every finite $T>0$ there exists $M_T<\infty$ such that
\begin{equation*}
  \sup_{t\in[0,T]} \sup_{x\in\R^n} \|\nabla_x s_t(x)\| \leq M_T.
\end{equation*}
In particular, the function
\begin{equation*}
  L(t) := \sup_{x\in\R^n}\|\nabla_x s_t(x)\|
\end{equation*}
is finite for all $t\in[0,T]$ and satisfies
\begin{equation*}
  \int_0^T (1+L(t))\,\d t \leq T(1+M_T) < \infty.
\end{equation*}
Furthermore, every element in $\mathcal{G}$ has finite second moment. Thus, the assumptions (i)–(ii) of Theorem~\ref{thm:score} hold for the initial density $p$ on any interval $[0,T]$. Since Lemma~\ref{lemma:score_prop_conv} provides boundedness of $L(t)$ also at $t=0$, by Corollary~\ref{cor:score} 
\begin{itemize}
  \item[(i)] the VP flow map $\varphi_{0\to T}$ is a bi-Lipschitz $C^1$-diffeomorphism 
  \item[(ii)] we have
        \begin{equation*}
          \|p - (\varphi_{0\to T}^{-1})_\# \gaussian\|_{L^1}
          \xrightarrow[T\to\infty]{} 0.
        \end{equation*}
\end{itemize}
Therefore, (ii) shows that there exists $T_\varepsilon>0$ such that for all $T\geq T_\varepsilon$,
\begin{equation*}
  \|p - (\varphi_{0\to T}^{-1})_\# \gaussian\|_{L^1} < \frac{\varepsilon}{2}.
\end{equation*}
We now fix such a $T\geq T_\varepsilon$.

\emph{Step 3. }
For this choice of $p\in\mathcal G$ and $T>0$, we combine the two $L^1$–bounds by the triangle inequality
\begin{equation*}
  \big\|\ptarget - (\varphi_{0\to T}^{-1})_\# \gaussian\big\|_{L^1}
  \leq \|\ptarget - p\|_{L^1}
     + \|p - (\varphi_{0\to T}^{-1})_\# \gaussian\|_{L^1}
  < \frac{\varepsilon}{2} + \frac{\varepsilon}{2}
  = \varepsilon.
\end{equation*}
Together with the bi-Lipschitz and score-regularity properties from Step~2, this proves the corollary.
\end{proof}

\section{Proofs of Section~\ref{sec:connection_learned}}\label{sec:appendix_learned}

This appendix contains the proof of the learned score result from Section~\ref{sec:connection_learned}.

\subsection{Proof of Theorem~\ref{thm:learned_score}}
\begin{proof}
    Let $\varepsilon>0$. We have
    \begin{align*}
        \| \ptarget - ((\varphi^{(m)}_{\delta \to T})^{-1})_\# \gaussian\|_{L^1}
        &\leq \underbrace{\| \ptarget - ((\varphi_{\delta\to T})^{-1})_\# \gaussian\|_{L^1}}_{=:\, \text{I}} \\
        &\qquad + \underbrace{\| ((\varphi_{\delta \to T})^{-1})_\# \gaussian
          - ((\varphi^{(m)}_{\delta\to T})^{-1})_\# \gaussian\|_{L^1}}_{=: \, \text{II}}.
    \end{align*}
    By Theorem~\ref{thm:score}~(II and III), we may choose $\delta>0$ and $T>\delta$ such that, simultaneously,
    \begin{equation*}
        \text{I}<\varepsilon/2\qquad\text{and}\qquad
        2\,\DKL{p_T}{\gaussian} < \frac{\varepsilon^2}{32}.
    \end{equation*}
    In particular, Pinsker's inequality yields $\|\gaussian-p_T\|_{L^1}< \varepsilon/(4\sqrt{2})<\varepsilon/4$.

    For II it holds
    \begin{align*}
        \text{II}
        &\leq \| ((\varphi_{\delta\to T})^{-1})_\# \gaussian
                - ((\varphi_{\delta \to T})^{-1})_\# p_T\|_{L^1}
            + \| ((\varphi_{\delta \to T})^{-1})_\# p_T
                - ((\varphi^{(m)}_{\delta \to T})^{-1})_\# \gaussian\|_{L^1}.
    \end{align*}
    By construction of the probability flow ODE and assumption~\emph{(iv)},
    \begin{equation*}
      ((\varphi_{\delta\to T})^{-1})_\# p_T = p_\delta,
      \qquad
      ((\varphi^{(m)}_{\delta\to T})^{-1})_\# \gaussian = q_\delta^{(m)},
    \end{equation*}
    where $p_\delta$ is the marginal of the true reverse-time dynamics at $t=\delta$ with terminal law $p_T$ and $q_\delta^{(m)}$ the marginal of the learned reverse-time dynamics with terminal law $\gaussian$. By invariance of the $L^1$-norm under bijective pushforwards, it holds
    \begin{equation*}
        \text{II} \leq \|\gaussian - p_T\|_{L^1} + \underbrace{\| p_\delta - q_\delta^{(m)}\|_{L^1}}_{=: \, \text{III}}.
    \end{equation*}
    We now show that
    \begin{equation*}
      \text{III} \leq \sqrt{\mathcal{E}_{\delta,T}(s^{(m)}) +2\,\DKL{p_T}{\gaussian}}.
    \end{equation*}
    Define the reverse-time SDE drifts
    \begin{equation*}
        b_t(x) = -\frac{1}{2}x - s_t(x),
        \qquad
        b_t^{(m)}(x) = -\frac{1}{2}x - s_t^{(m)}(x).
    \end{equation*}
    Then, $P$ is the path law on $C([\delta,T];\R^n)$ of
    \begin{equation*}
      \mathrm{d}X_t = b_t(X_t)\,\mathrm{d}t + \mathrm{d}W_t,
    \end{equation*}
    with terminal law $X_T\sim p_T$. Furthermore, $Q^{(m)}$ is the path law on $C([\delta,T];\R^n)$ of
    \begin{equation}\label{eq:b_t_sde}
      \mathrm{d}X_t^{(m)} = b_t^{(m)}(X_t^{(m)})\,\mathrm{d}t + \mathrm{d}W_t,
    \end{equation}
    with terminal law $X_T^{(m)}\sim \gaussian$ and $\overline Q^{(m)}$ the path law of~\eqref{eq:b_t_sde} but with terminal law $\overline X_T^{(m)}\sim p_T$. Since
    \begin{equation*}
        b_t(X_t) - b^{(m)}_t(X_t)
        = \Bigl(-\frac{1}{2}X_t - s_t(X_t)\Bigr)
          - \Bigl(-\frac{1}{2}X_t - s^{(m)}_t(X_t)\Bigr)
        = s^{(m)}_t(X_t) - s_t(X_t),
    \end{equation*} 
    Assumption~\emph{(ii)} guarantees the Novikov condition and absolute continuity needed for Girsanov, so by Lemma~\ref{lemma:girsanov_KL},
    \begin{align*}
        \DKL{P}{\overline Q^{(m)}} &= \frac{1}{2} \E_P\left( \int_\delta^T
           \|s_t(X_t) - s^{(m)}_t(X_t) \|^2 \, \mathrm{d}t\right)
    \end{align*}
    Therefore, we obtain via Tonelli
    \begin{equation*}
        \DKL{P}{\overline Q^{(m)}}= \frac{1}{2} \int_\delta^T \int_{\R^n} p_t(x)
             \|s_t(x)-s^{(m)}_t(x)\|^2 \, \mathrm{d}x \,  \mathrm{d}t = \frac{1}{2}\mathcal{E}_{\delta, T}(s^{(m)}).
    \end{equation*}
    Moreover, since $\overline Q^{(m)}$ and $Q^{(m)}$ have the same conditional path law given $X_T$, the chain rule for the KL-divergence implies
    \begin{equation*}
        \DKL{P}{Q^{(m)}} = \DKL{P}{\overline Q^{(m)}} + \DKL{p_T}{\gaussian}.
    \end{equation*}
    Let $\Phi: C([\delta,T];\R^n) \rightarrow \R^n$ be given by $\Phi(\omega) =\omega(\delta)$. Then $\Phi_\#P = p_\delta$ and $\Phi_\# Q^{(m)} = q_\delta^{(m)}$. By the data processing inequality for the KL-divergence, the KL-divergence decreases under measurable maps, i.e.,
    \begin{equation*}
        \DKL{p_\delta}{q_\delta^{(m)}} = \DKL{\Phi_\# P}{\Phi_\# Q^{(m)}} \leq \DKL{P}{Q^{(m)}}.
    \end{equation*}
    Combining the previous estimates yields
    \begin{equation*}
        \DKL{p_\delta}{q_\delta^{(m)}}
        \leq \frac{1}{2}\mathcal{E}_{\delta,T}(s^{(m)}) + \DKL{p_T}{\gaussian}.
    \end{equation*}
    Together with Pinsker's inequality it holds
    \begin{equation*}
        \text{III} = \| p_\delta - q_\delta^{(m)}\|_{L^1} \leq \sqrt{2\DKL{p_\delta}{q_\delta^{(m)}} } \leq \sqrt{\mathcal{E}_{\delta,T}(s^{(m)}) +2\,\DKL{p_T}{\gaussian}}.
    \end{equation*}
    By Assumption~\emph{(iii)}, $\mathcal{E}_{\delta,T}(s^{(m)}) \to 0$ as $m\to\infty$ for this fixed interval $[\delta,T]$. Therefore, we can select $m_0 \in \N$ such that $\mathcal{E}_{\delta,T}(s^{(m_0)}) \leq \varepsilon^2/32$. With our choices for $\delta$ and $T$,
    \begin{equation*}
      \text{III}
      \leq \sqrt{\mathcal{E}_{\delta,T}(s^{(m_0)}) + 2\,\DKL{p_T}{\gaussian}}
      < \sqrt{\frac{\varepsilon^2}{32}+\frac{\varepsilon^2}{32}}
      = \frac{\varepsilon}{4}.
    \end{equation*}
     Our choices above imply
    \begin{equation*}
      \text{II} \leq \|\gaussian - p_T\|_{L^1} + \text{III} < \frac{\varepsilon}{4} + \frac{\varepsilon}{4}  = \frac{\varepsilon}{2}.
    \end{equation*}
    Altogether,
    \begin{equation*}
        \| \ptarget - ((\varphi^{(m_0)}_{\delta \to T})^{-1})_\# \gaussian \|_{L^1}
        \leq \text{I} + \text{II} < \frac{\varepsilon}{2} + \frac{\varepsilon}{2} = \varepsilon.
    \end{equation*}
    Finally, the bi-Lipschitz property of $\varphi^{(m)}_{\delta\to T}$ for each $m$ follows from Assumption~\emph{(i)} by the same argument as in the proof of Theorem~\ref{thm:score}.
\end{proof}

\section{Auxiliary Results from Section~\ref{sec:expressivity}}
\label{sec:appendix}

This section collects auxiliary results for Appendix~\ref{sec:appendix}, specifically required for the proofs of Theorem~\ref{thm:score} and Corollary~\ref{cor:score}.

The representation of $p_t$ as a Mehler kernel convolution and the smoothing properties of the Ornstein–Uhlenbeck semigroup are classical, see e.g.~\cite{BakryGentilLedoux2014}. 
For completeness, we provide a direct proof of joint continuity and differentiation under the integral at the level of densities.

\begin{lemma}\label{lem:regularity_pt}
Let $\ptarget:\R^n\to[0,\infty)$ be a pdf with $\E_{\ptarget}[\|X\|^2] < \infty$. 
Consider the VP diffusion process with $\beta\equiv1$ and initial law $X_0\sim\ptarget$.
Then, for $t > 0$ we can write
\begin{equation*}
    p_t(x) = \int_{\R^n} \ptarget(y) \phi_{\sigma^2(t),0}(x-a(t)y)\,\d y,
  \qquad 
  a(t)=e^{-t/2},\quad \sigma^2(t)=1-e^{-t}.
\end{equation*}
Fix $0<\delta<T<\infty$. Then:
\begin{itemize}
  \item[(i)] The map $(t,x)\mapsto p_t(x)$ is continuous on $[\delta,T]\times\R^n$.
  \item[(ii)] For every $t>0$ and $x\in\R^n$, $\nabla_x p_t(x)$ exists and is
        given by
        \begin{equation*}
            \nabla_x p_t(x)
          = \int_{\R^n} \ptarget(y)\,\nabla_x\phi_{\sigma^2(t),0}(x-a(t)y)\,\d y,
        \end{equation*}
        with
        \begin{equation*}
            \nabla_x\phi_{\sigma^2(t),0}(x-a(t)y)
          = -\frac{x-a(t)y}{\sigma^2(t)}\,\phi_{\sigma^2(t),0}(x-a(t)y).
        \end{equation*}
  \item[(iii)] The map $(t,x)\mapsto \nabla_x p_t(x)$ is continuous on $[\delta,T]\times\R^n$.
  \item[(iv)] For every $t>0$ and $x\in\R^n$, $\nabla_x^2 p_t(x)$ exists and is
        given by
        \begin{equation*}
            \nabla_x^2 p_t(x)
          = \int_{\R^n} \ptarget(y)\,\nabla_x^2\phi_{\sigma^2(t),0}(x-a(t)y)\,\d y,
        \end{equation*}
        where
        \begin{equation*}
            \nabla_x^2\phi_{\sigma^2(t),0}(x-a(t)y)
          = \left(\frac{(x-a(t)y)(x-a(t)y)^\top}{\sigma^4(t)}-\frac{1}{\sigma^2(t)}I_n\right)
            \phi_{\sigma^2(t),0}(x-a(t)y).
        \end{equation*}
        Moreover, $(t,x)\mapsto \nabla_x^2 p_t(x)$ is continuous on $[\delta,T]\times\R^n$.
  \item[(v)] The score $s_t(x)=\nabla_x\log p_t(x)$ is well-defined and continuous on $[\delta,T]\times\R^n$, and
        $(t,x)\mapsto \nabla_x s_t(x)$ is continuous on $[\delta,T]\times\R^n$.
        In particular,
        \begin{equation*}
            v_t(x)= -\tfrac12 x - \tfrac12 s_t(x)
        \end{equation*}
        is $C^1$ in $x$ for each $t\in[\delta,T]$, and $(t,x)\mapsto \nabla_x v_t(x)$ is continuous on $[\delta,T]\times\R^n$.
\end{itemize}

If additionally, $\ptarget \in C^2_b(\R^n)$ and $\ptarget > 0$, all statements on $[\delta,T] \times \R^n$ extend to $[0,T] \times \R^n$.
\end{lemma}

\begin{proof}
For $t\in[\delta,T]$ we have $1 \geq \sigma^2(t) \geq\sigma^2(\delta)>0$, hence the Gaussian kernel
${(t,x,y)\mapsto \phi_{\sigma^2(t),0}(x-a(t)y)}$ and its $x$-derivatives up to order $2$ are
continuous in $(t,x)$ and satisfy, for each fixed compact $K\subset\R^n$,
uniform bounds of the form
\begin{equation*}
    \big|\partial_x^\alpha \phi_{\sigma^2(t),0}(x-a(t)y)\big| \leq C_{\delta,T,K}\,(1+\|y\|^2)\,\phi_{\tilde\sigma^2,0}(y),  \qquad |\alpha|\leq 2
\end{equation*}
for some $\tilde\sigma^2>0$ and $(t,x) \in [\delta,T] \times K$. Since $\int_{\R^n} \ptarget(y)(1+\|y\|^2)\,\d y<\infty$, dominated convergence yields \emph{(i)} and justifies differentiation under the integral to obtain the formulas in \emph{(ii)} and \emph{(iv)}, as well as the joint continuity in \emph{(iii)} and \emph{(iv)}. (These are standard smoothing properties of the OU/Mehler kernel, see e.g. \cite{BakryGentilLedoux2014} for a semigroup-based presentation.)

For (v), note that $p_t>0$ for $t>0$, hence 
\begin{equation*}
    s_t(x) = \nabla_x\log p_t(x) = \frac{\nabla_x p_t(x)}{p_t(x)}
\end{equation*} 
is well-defined. Moreover,
\begin{equation*}
    \nabla_x s_t(x)=\nabla_x^2\log p_t(x)
    =\frac{\nabla_x^2 p_t(x)}{p_t(x)}-\frac{\nabla_x p_t(x)\,\nabla_x p_t(x)^\top}{p_t(x)^2},
\end{equation*}
so the continuity of $p_t,\nabla_x p_t,\nabla_x^2 p_t$ on $[\delta,T]\times\R^n$ implies the stated
continuity of $s_t$ and $\nabla_x s_t$, and thus the claims for $v_t$.

\emph{Extension to $[0,T]\times\R^n$ for $\ptarget\in C^2_b(\R^n)$ and $\ptarget>0$:} We can write
\begin{equation*}
    p_t = (D_{a(t)}\ptarget)\ast\phi_{\sigma^2(t),0},\qquad (D_a p)(x)=a^{-n}p(x/a),
\end{equation*}
with $a(t)\to 1$ and $\sigma(t)\to 0$ as $t\downarrow 0$.

\emph{(i): Continuity of $(t,x)\mapsto p_t(x)$ on $[0,T]\times\R^n$.}
Fix a compact $K\subset\R^n$. Choose $t_0>0$ such that $a(t)\in[1/2,1]$ for all $t\in(0,t_0]$, and set
\begin{equation*}
    K':=\{x/a(t):\,x\in K,\ t\in(0,t_0]\},
\end{equation*}
which is compact. Hence, for $x\in K$,
\begin{equation*}
    |D_{a(t)}\ptarget(x)-\ptarget(x)|
%=|a(t)^{-n}\ptarget(x/a(t))-\ptarget(x)|
\leq |a(t)^{-n}-1|\,|\ptarget(x/a(t))|+|\ptarget(x/a(t))-\ptarget(x)|.
\end{equation*}
Since $\ptarget$ is continuous, it is bounded and uniformly continuous on $K'$. Taking the supremum over $x\in K$ yields
\begin{equation*}
    \sup_{x\in K}|D_{a(t)}\ptarget(x)-\ptarget(x)|
\leq |a(t)^{-n}-1|\sup_{u\in K'}|\ptarget(u)|
+\sup_{x\in K}|\ptarget(x/a(t))-\ptarget(x)| \xrightarrow{t \downarrow 0} 0.
\end{equation*}
The first term tends to $0$ as $t\downarrow 0$ since $a(t)\to 1$ and $\ptarget$ is bounded on $K'$.
The second term tends to 0 as $x/a(t)\to x$ uniformly for $x\in K$. Consequently, $D_{a(t)}\ptarget\to\ptarget$ locally uniformly on $\R^n$.

Now write
\begin{equation*}
    p_t-\ptarget
=\bigl(D_{a(t)}\ptarget-\ptarget\bigr)\ast\phi_{\sigma^2(t),0}
+\bigl(\ptarget*\phi_{\sigma^2(t),0}-\ptarget\bigr).
\end{equation*}
For the first term, fix $r>0$ and define $K_r:=\{x-z:\,x\in K,\ \|z\|\leq r\}$. Then for $x\in K$,
\begin{align*}
\left| \bigl((D_{a(t)}\ptarget-\ptarget)*\phi_{\sigma^2(t),0}\bigr)(x) \right| &\leq \int_{\| z\| \leq r}|D_{a(t)}\ptarget(x-z)-\ptarget(x-z)| \phi_{\sigma^2(t),0}(z) \d z \\
&\quad+ \int_{\| z\| > r}|D_{a(t)}\ptarget(x-z)-\ptarget(x-z)| \phi_{\sigma^2(t),0}(z) \d z \\
&\leq \sup_{u\in K_r}|D_{a(t)}\ptarget(u)-\ptarget(u)|
   + 2\sup_{\|z\|>r}\phi_{\sigma^2(t),0}(z),
\end{align*}
where we used $\|D_{a(t)}\ptarget\|_{L^1}=\|\ptarget\|_{L^1}=1$ for the tail bound.
Since $D_{a(t)}\ptarget\to\ptarget$ locally uniformly, the first term tends to $0$ as $t\downarrow 0$. Moreover, for any fixed $r>0$,
$\sup_{\|z\|>r}\phi_{\sigma^2(t),0}(z)=(2\pi \sigma^2(t))^{-n/2}\exp(-r^2/(2\sigma^2(t)))\to 0$ as $\sigma(t)\to 0$.
Hence
\begin{equation*}
    \sup_{x\in K}\left| \bigl((D_{a(t)}\ptarget-\ptarget)\ast\phi_{\sigma^2(t),0}\bigr)(x) \right|
\xrightarrow{t \downarrow0} 0.
\end{equation*}

For the second term, since $\ptarget\in C \cap L^\infty$, standard Gaussian approximate identity properties give
\begin{equation*}
    \sup_{x\in K}|\ptarget\ast\phi_{\sigma^2(t),0}(x)-\ptarget(x)|\xrightarrow{t \downarrow0} 0.
\end{equation*}
Combining both terms yields $\sup_{x\in K}|p_t(x)-\ptarget(x)|\to 0$, hence continuity at $t=0$.
Continuity on $(0,T]\times\R^n$ is already proved above.

\emph{(ii)--(iv): Continuity of $\nabla_x p_t$ and $\nabla_x^2 p_t$ on $[0,T]\times\R^n$.}
Since $\ptarget\in C^2$, we have $p_0=\ptarget$, $\nabla p_0=\nabla\ptarget$, and $\nabla^2 p_0=\nabla^2\ptarget$.
Moreover, for $k\in\{1,2\}$,
\begin{equation*}
    \nabla^k p_t = (\nabla^k D_{a(t)}\ptarget)\ast\phi_{\sigma^2(t),0},
\qquad
\nabla^k(D_a\ptarget)(x)=a^{-n-k}(\nabla^k\ptarget)(x/a).
\end{equation*}
Applying the argument from \emph{(i)} to the continuous functions $\nabla\ptarget, \, \nabla^2\ptarget \in C \cap L^\infty$, gives local uniform convergence $\nabla p_t\to\nabla\ptarget$ and $\nabla^2 p_t\to\nabla^2\ptarget$ as $t\downarrow 0$, hence continuity on $[0,T]\times\R^n$.

\emph{(v): Continuity of $s_t,\nabla s_t$ and $v_t,\nabla v_t$ on $[0,T]\times\R^n$.}
Fix a compact $K\subset\R^n$. Since $\ptarget>0$ and continuous, $c_K:=\inf_{x\in K}\ptarget(x)>0$.
By \emph{(i)}, $p_t\to\ptarget$ locally uniformly, hence $\inf_{x\in K}p_t(x)\ge c_K/2$ for all $t$
sufficiently small. Therefore, on $K$,
\begin{equation*}
    s_t=\nabla\log p_t=\frac{\nabla p_t}{p_t} \quad\text{and}\quad \nabla s_t=\nabla^2\log p_t
=\frac{\nabla^2 p_t}{p_t}-\frac{\nabla p_t\,\nabla p_t^\top}{p_t^2}
\end{equation*}
are continuous at $t=0$ by \emph{(i)--(iv)} and the uniform lower bound on $p_t$.
Finally, $v_t(x)=-\tfrac12 x-\tfrac12 s_t(x)$ inherits continuity of $s_t$ and $\nabla s_t$.
\end{proof}

\begin{lemma}\label{lemma:L1_conv_pt}
Let $p\in L^1$ and define the dilation operator
\begin{equation*}
(D_a p)(x):=a^{-n}p(x/a),\qquad a>0.
\end{equation*}
If $a(t)\to 1$ and $\sigma(t)\to 0$ as $t\downarrow 0$, then
\begin{equation*}
\bigl\|D_{a(t)}p \ast \phi_{\sigma(t)^2,0}-p\bigr\|_{L^1}
\xrightarrow[t\downarrow0]{}0.
\end{equation*}
\end{lemma}

\begin{proof}
We use the decomposition
\begin{equation*}
\|D_{a(t)}p \ast \phi_{\sigma(t)^2,0}-p\|_{L^1}\leq \|(D_{a(t)}p-p)\ast\phi_{\sigma(t)^2,0}\|_{L^1}+\|p\ast\phi_{\sigma(t)^2,0}-p\|_{L^1}.
\end{equation*}
By Young's convolution inequality and $\|\phi_{\sigma^2,0}\|_{L^1}=1$,
\begin{equation*}
\|(D_{a(t)}p-p)\ast\phi_{\sigma(t)^2,0}\|_{L^1} \leq \|D_{a(t)}p-p\|_{L^1}.
\end{equation*}
Hence it suffices to show
\begin{equation*}
\text{I:} \ \|D_{a}p-p\|_{L^1}\xrightarrow[a\to1]{}0
\qquad \text{and} \qquad
\text{II:} \ \|p\ast\phi_{\sigma^2,0}-p\|_{L^1}\xrightarrow[\sigma\to0]{}0.
\end{equation*}

\emph{I:}
The operator $D_a:L^1(\R^n)\to L^1(\R^n)$ is an isometry, i.e. $\|D_a p\|_{L^1}=\|p\|_{L^1}.$
Fix $\varepsilon>0$ and since $C_c^\infty(\R^n)$ is dense in $L^1(\R^n)$, choose $g\in C_c^\infty(\R^n)$ with
\begin{equation*}
\|p-g\|_{L^1}<\varepsilon.
\end{equation*}
Then
\begin{equation*}
\|D_ap-p\|_{L^1}\leq \|D_a(p-g)\|_{L^1}+\|D_ag-g\|_{L^1}+\|g-p\|_{L^1}\leq 2\varepsilon+\|D_ag-g\|_{L^1}.
\end{equation*}
Since $g$ is continuous with compact support, we have
\begin{equation*}
D_ag(x)=a^{-n}g(x/a)\rightarrow g(x)
\end{equation*}
pointwise as $a\to1$. For $a$ close to $1$, the functions $D_ag$ are supported in a fixed compact set and uniformly bounded, so dominated convergence yields
\begin{equation*}
\|D_ag-g\|_{L^1}\xrightarrow[a\to1]{}0.
\end{equation*}
Since $\varepsilon>0$ was arbitrary, this proves the first limit.

\emph{II:}
Let $\tau_y p(x):=p(x-y)$. Then,
\begin{equation*}
p\ast\phi_{\sigma^2,0}(x)-p(x) = \int_{\R^n}\phi_{\sigma^2,0}(y)\bigl(\tau_y p(x)-p(x)\bigr)\,\d y,
\end{equation*}
and hence
\begin{equation*}
\|p\ast\phi_{\sigma^2,0}-p\|_{L^1} \leq \int_{\R^n}\phi_{\sigma^2,0}(y)\,\|\tau_y p-p\|_{L^1}\,\d y.
\end{equation*}
Fix $\eta>0$. Since translations are continuous on $L^1$, there exists $\bar\eta>0$ such that
\begin{equation*}
\|\tau_y p-p\|_{L^1}<\eta \quad\text{whenever }\|y\|<\bar\eta.
\end{equation*}
Moreover,
\begin{equation*}
\|\tau_y p-p\|_{L^1}\leq 2\|p\|_{L^1} \quad\text{for all } y.
\end{equation*}
Therefore,
\begin{equation*}
\|p\ast\phi_{\sigma^2,0}-p\|_{L^1}\leq
\eta \int_{\|y\|<\bar\eta}\phi_{\sigma^2,0}(y)\,\d y + 2\|p\|_{L^1}\int_{\|y\|\geq\bar\eta}\phi_{\sigma^2,0}(y)\,\d y.
\end{equation*}
As $\sigma\to0$, the Gaussian measure concentrates at $0$, so
\begin{equation*}
\int_{\|y\|\geq\bar\eta}\phi_{\sigma^2,0}(y)\,\d y \xrightarrow[\sigma\to0]{}0.
\end{equation*}
Hence
\begin{equation*}
\limsup_{\sigma\to0}\|p\ast\phi_{\sigma^2,0}-p\|_{L^1}\leq \eta.
\end{equation*}
Since $\eta>0$ was arbitrary, the second limit follows.

Combining the above estimates and using $a(t)\to1$ and $\sigma(t)\to0$ yields
\begin{equation*}
\bigl\|D_{a(t)}p\ast\phi_{\sigma(t)^2,0}-p\bigr\|_{L^1} \xrightarrow[t\downarrow0]{}0.
\end{equation*}
\end{proof}

\section{Auxiliary Results for Section~\ref{sec:connection_learned}}

This section contains an auxiliary Girsanov type result used in Appendix~\ref{sec:appendix_learned}, in the proof of Theorem~\ref{thm:learned_score}. The main proof ingredients for this result can be taken from existing literature, as e.g.~\cite{Oksendal2003}.

\begin{lemma}\label{lemma:girsanov_KL}
Fix $0 < \delta < T < \infty$ and consider the SDE
\begin{equation*}
    \d X_t = b^{(m)}_t(X_t)\,\d t + \d W_t, \qquad t\in[\delta,T],
\end{equation*}
where $b^{(m)}_t:\R^n\to\R^n$ is a measurable drift and $W$ is an $n$-dimensional Brownian motion. Let $b_t:\R^n\to\R^n$ be another measurable drift and define
\begin{equation*}
    u^{(m)}_t := b^{(m)}_t(X_t) - b_t(X_t), \qquad t\in[\delta,T].
\end{equation*}
Assume Novikov's condition
\begin{equation*}
    \E_{\overline Q^{(m)}}\left[\exp\Big(\frac12\int_\delta^T \|u^{(m)}_s\|^2\,\d s\Big)\right] < \infty,
\end{equation*}
is satisfied, where $\overline Q^{(m)}$ denotes the law of $(X_t)_{\delta\leq t\leq T}$.
Define the path law $P$ via
\begin{equation*}
    \frac{\d P}{\d \overline Q^{(m)}} := \exp\left(
        - \int_\delta^T \langle u^{(m)}_s,\d W_s\rangle
        - \frac12 \int_\delta^T \|u^{(m)}_s\|^2\,\d s \right).
\end{equation*}
Then the Kullback-Leibler divergence between the path laws satisfies
\begin{equation*}
    \DKL{P}{\overline Q^{(m)}}
    = \frac12 \E_{P}\left[
        \int_\delta^T \|b^{(m)}_t(X_t) - b_t(X_t)\|^2\,\d t
    \right].
\end{equation*}
\end{lemma}

\begin{proof}
Due to Novikov's condition, Girsanov's theorem~\cite[Theorem~8.6.4]{Oksendal2003} implies that the process
\begin{equation*}
    W^{(m)}_t := W_t + \int_\delta^t u^{(m)}_s\,\d s
\end{equation*}
is an $n$-dimensional Brownian motion under $P$ and $X_t$ satisfies
\begin{equation*}
    \mathrm d X_t = b_t(X_t)\,\d t + \d W^{(m)}_t, \qquad t\in[\delta,T],
\end{equation*}
under $P$.
Moreover, using $W_t = W^{(m)}_t - \int_\delta^t u^{(m)}_s\,\d s$ and linearity of the It{\^o} integral (see~\citeauthor{Oksendal2003} \citeyear[Theorem~3.2.1]{Oksendal2003}), applied componentwise to each coordinate of the multidimensional Brownian motion, we obtain
\begin{align*}
    \log\frac{\d P}{\d Q^{(m)}}
    &= -\int_\delta^T \langle u^{(m)}_s, \d W_s\rangle
      - \frac12 \int_\delta^T \|u^{(m)}_s\|^2\,\d s \\
    &= -\int_\delta^T \langle u^{(m)}_s,\d W^{(m)}_s\rangle
      + \int_\delta^T \|u^{(m)}_s\|^2\,\mathrm d s - \frac12 \int_\delta^T \|u^{(m)}_s\|^2\,\d s \\
    &= -\int_\delta^T \langle u^{(m)}_s,\d W^{(m)}_s\rangle
      + \frac12 \int_\delta^T \|u^{(m)}_s\|^2\,\d s.
\end{align*}
Taking expectations under $P$ yields
\begin{align*}
    \DKL{P}{\overline Q^{(m)}}
    &= \E_{P}\left[
        \log\frac{\d P}{\d \overline Q^{(m)}}
    \right] \\
    &= - \E_{P}\left[
        \int_\delta^T \langle u^{(m)}_s,\d W^{(m)}_s\rangle
    \right]
       + \frac12 \E_{P}\left[
           \int_\delta^T \|u^{(m)}_s\|^2\,\d s
       \right].
\end{align*}
The stochastic integral $\int_\delta^t \langle u^{(m)}_s,\d W^{(m)}_s\rangle$ has zero expectation under $P$ for each $t$ (see~\citeauthor{Oksendal2003} \citeyear[Theorem~3.2.1]{Oksendal2003}), hence
\begin{equation*}
    \DKL{P}{\overline Q^{(m)}}
    = \frac12 \E_{P}\left[
        \int_\delta^T \|u^{(m)}_s\|^2\,\d s
    \right].
\end{equation*}
\end{proof}

\printbibliography

\end{document}